\documentclass[conference]{IEEEtran}
\usepackage{times}

\usepackage{amsmath,amssymb,amsthm}

\usepackage[table]{xcolor}
\usepackage{graphicx}
\usepackage{arydshln}
\usepackage{subcaption}
\usepackage[]{mathtools}
\usepackage[numbers,compress]{natbib}
\usepackage{algorithmicx}
\usepackage{algorithm}
\usepackage{soul}
\usepackage[]{algpseudocode}

\algrenewcommand\algorithmicindent{0.5em}
\algnewcommand\algorithmicinput{\textbf{Input:}}
\algnewcommand\INPUT{\item[\algorithmicinput]}
\usepackage[bookmarksopen,bookmarksnumbered,colorlinks=true, linkcolor=blue,
citecolor = blue, urlcolor = cyan, filecolor=black , pagebackref=false,hypertexnames=false, plainpages=false, pdfpagelabels]{hyperref}
\usepackage{comment}
\usepackage{multirow}

\let\labelindent\relax
\usepackage{enumitem}

\usepackage{tikz} 
\usepackage{tkz-graph}
\usepackage{tkz-berge}
\usetikzlibrary{backgrounds,fit,shapes,snakes,arrows,shapes.geometric,positioning}
\usetikzlibrary{intersections,patterns,shapes.misc}
\usetikzlibrary{decorations.pathmorphing}

\tikzstyle{block} = [rectangle, rounded corners, minimum width=3cm, minimum height=1cm,text centered, draw=black, fill=red!30]
\tikzstyle{new} = [rectangle, rounded corners, minimum width=1cm, minimum
height=1cm,text centered, draw=black, fill=blue!10!white, dashed]
\tikzstyle{arrow} = [thick,->,>=stealth]
\usetikzlibrary{calc, quotes}

\tikzstyle{fblock} = [rectangle, draw, fill=gray!20, 
text width=8em, text centered, rounded corners, minimum height=4em, minimum width = 8em]
\tikzstyle{line} = [draw, -latex']

\usepackage{fontawesome}
\DeclareFontFamily{OT1}{pzc}{}
\DeclareFontShape{OT1}{pzc}{m}{it}{<-> s * [1.200] pzcmi7t}{}
\DeclareMathAlphabet{\mathpzc}{OT1}{pzc}{m}{it}

\AtBeginDocument{%
  \DeclareMathAlphabet\PazoBB{U}{fplmbb}{m}{n}%
}

\usepackage{dsfont}
\usepackage{microtype}

\usepackage{diagbox}
\usepackage{multirow}
\usepackage{dashbox}

\usepackage{sidecap}
\usepackage{soul}

\theoremstyle{plain}
\newtheorem{theorem}{Theorem}
\newtheorem{lemma}{Lemma}
\newtheorem{proposition}{Proposition}

\theoremstyle{definition}

\newtheorem{assumption}{Assumption}
\newtheorem{problem}{Problem}
\newtheorem{remark}{Remark}

\DeclareMathOperator{\SOd}{SO}
\DeclareMathOperator*{\argmin}{argmin}
\DeclareMathOperator*{\argmax}{argmax}

\newcommand{\Rset}{\mathbb{R}}

\newcommand{\Mcal}{\mathcal{M}}

\newcommand{\Zcal}{\mathcal{Z}}

\newcommand{\Ncal}{\mathcal{N}}
\newcommand{\Ycal}{\mathcal{Y}}
\newcommand{\Wcal}{\mathcal{W}}

\newcommand{\defremove}[1]{}

\newcommand{\xtrue}{x_\text{true}}
\newcommand{\Sigtrue}{\Sigma_\text{true}}
\newcommand{\Ptrue}{P_\text{true}}

\newcommand{\xmle}{x^\star}

\newcommand{\problems}{Problems~\ref{prob:joint}-\ref{prob:joint_diagonal_constrained}}

\IEEEoverridecommandlockouts
\begin{document}

\title{Joint State and Noise Covariance Estimation}

\author{\authorblockN{Kasra Khosoussi$^\ast$\thanks{$^\ast$Also affiliated with 	
		CSIRO Robotics, Data61.}}
\authorblockA{School of Electrical Engineering and Computer Science\\
The University of Queensland\\
St Lucia, QLD, Australia\\
k.khosoussi@uq.edu.au}
\and
\authorblockN{Iman Shames}
\authorblockA{School of Engineering\\
The Australian National University\\
Canberra, ACT, Australia\\
iman.shames@anu.edu.au}}

\maketitle

\begin{abstract}
		This paper tackles the problem of jointly estimating the noise covariance matrix
		alongside states (parameters such as poses and points) from measurements
		corrupted by Gaussian noise and, if available, prior information. In such settings, the noise covariance matrix
		determines the weights assigned to individual measurements in the least squares
		problem. We show that the joint problem exhibits a convex structure and provide a
		full characterization of the optimal noise covariance estimate (with analytical
		solutions) within joint maximum \emph{a posteriori} and likelihood
		frameworks and several variants.
		Leveraging this theoretical result, we propose two novel algorithms that jointly
		estimate the primary parameters and the noise covariance matrix. 
		Our BCD algorithm can be easily integrated into existing nonlinear least
		squares solvers, with negligible per-iteration computational overhead.
		To validate our
		approach, we conduct extensive experiments across diverse scenarios and offer
		practical insights into their application in robotics and computer vision
		estimation problems with a particular focus on SLAM.
\end{abstract}

\IEEEpeerreviewmaketitle

\section{Introduction}
\label{sec:introduction}
Maximum likelihood (ML) and maximum \emph{a
posteriori} (MAP) are the two most common
point-estimation criteria in robotics and computer vision applications such as
all variants of simultaneous localization and mapping (SLAM) and bundle adjustment.
Under the standard assumption of zero-mean Gaussian measurement noise---and,
for MAP, Gaussian priors---these estimation problems
reduce to least squares; i.e., 
finding estimates that minimize the weighted sum of squared errors between
observed and expected measurements (and, when priors are available, the weighted squared
errors between parameter values and their expected priors).

Each measurement error in the least squares objective is weighted by its corresponding noise information
matrix (i.e., the inverse of the covariance matrix). Intuitively, 
more precise sensors receive ``larger'' weights, thus exerting greater
influence on the final estimate. Moreover, the weight {matrices} enable the estimator to account for correlations
between different components of measurements, preventing ``double counting''
of information.
Therefore, obtaining an accurate estimate of the noise covariance matrix is
critical for achieving high estimation accuracy.
In addition, the (estimated) noise covariance matrix also determines the (estimated) covariance of the 
estimated parameter values often used for control and decision making in robotics. 
Consequently, an inaccurate noise covariance estimate can cause 
overconfidence or underconfidence in state estimates, potentially leading to
poor or even catastrophic decisions.

In principle, the noise covariance matrix can be estimated \emph{a priori}
(offline) using a calibration dataset where the true values of the primary
parameters (e.g., robot poses) are known (see Remark~\ref{rem:calib}).
This can be done either by the sensor manufacturer or the end user.
However, in practice, several challenges arise: 
\begin{enumerate}[leftmargin=*]
		\item Calibration is a
				labor-intensive process and may not always be feasible, particularly when
				obtaining ground truth for primary parameters requires additional
				instrumentation.
		\item  In many cases, raw measurements are preprocessed by intermediate
				algorithms before being used in the estimation problem (e.g., in
				a SLAM front-end), making it difficult to model their noise
				characteristics. 
		\item The noise characteristics may evolve over time (e.g., due to
				dynamic environmental factors such as temperature), making the
				pre-calibrated noise model obsolete.  
\end{enumerate}
Due to these challenges, many applications 
rely on ad hoc noise covariances, such as arbitrary isotropic (or
diagonal) covariances, which are either manually set by experts or determined
through trial-and-error tuning.
Despite being recognized as one of the most critical and widely
acknowledged challenges in SLAM \cite[Sections III.B,
III.G, and V]{ebadi2023present}, the problem of noise covariance estimation remains unsolved and
understudied in
robotics and computer vision literature.

We present, to the best of our knowledge, the first algorithms for online (i.e.,
during deployment) \emph{joint} ML/MAP estimation of states and noise covariance
matrices from noisy measurements and, when available, prior information. Our
approach is general and eliminates the need for a separate calibration stage
across a broad class of estimation problems beyond SLAM. We analyze the
convergence properties of the proposed algorithm and demonstrate that it can be
seamlessly integrated into existing sparse nonlinear least squares
solvers~\cite{Agarwal_Ceres_Solver_2022,gtsam,kuemmerle2011g2o}, with negligible
computational overhead.

\subsection*{Notation}
We use $[n]$ to denote the set of integers from $1$ to $n$. The abbreviated
notation $x_{1:n}$ is used to denote $x_1, \ldots, x_n$.
The zero matrix (and vector) is denote by $0$ where the size should be clear
from the context.
$\mathbb{S}^d_{\succeq 0}$ and $\mathbb{S}^d_{\succ 0}$ denote the sets of $d \times d$
symmetric positive semidefinite and positive definite real matrices,
respectively. For two symmetric real matrices $A$ and $B$, $A \succeq B$
(resp., $A \succ B$) means $A-B$ is positive semidefinite (resp.,
positive definite). $A_{ij}$ denotes the $(i,j)$ element of matrix $A$, and
$\mathrm{Diag}(A)$ denotes the diagonal matrix obtained by zeroing out the
off-diagonal elements of $A$. The standard (Frobenius) inner product between
$n \times n$ real matrices $A$ and $B$ is denoted by $\langle A,B \rangle \triangleq \mathrm{trace}(A^\top
B)$. The Frobenius norm of $A$ is denoted by $\|A\| = \sqrt{\langle
A,A\rangle}$. 
The weighted Euclidean norm of $x$ given a weight matrix $W \succ 0$ is
denoted by $\|x\|_W \triangleq \sqrt{x^\top W x}$. The probability density
function of the multivariate normal
distribution of random variable $x$ with mean vector $\mu$ and covariance matrix $\Sigma$ is denoted by
$\Ncal(x; \mu,\Sigma)$.

\section{Related Works}
\label{sec:related}
We refer the reader to
\cite{cadena2016simultaneous,dellaert2017factor,ebadi2023present,triggs1999bundle}
for comprehensive reviews of state-of-the-art estimation frameworks in robotics
and computer vision. Sparse nonlinear least squares solvers for
solving these estimation problems can
be found in \cite{Agarwal_Ceres_Solver_2022,gtsam,kuemmerle2011g2o}.
However, all of these works assume that the noise covariance is known
beforehand. In contrast, our approach simultaneously estimates both the primary
parameters (e.g., robot's trajectory) and the noise covariance matrix
directly from noisy measurements.
The importance of automatic hyperparameter tuning has recently gained
recognition in the SLAM literature; see, e.g., \cite[Section V]{ebadi2023present}
and \cite{fontan2024look,fontan2024anyfeature}. We share this perspective and
present, to the best of our knowledge, the first principled
approach for ML/MAP measurement covariance estimation in SLAM and related
problems.

\subsection{Optimal Covariance Estimation via Convex Optimization}
ML estimation of the mean and covariance from independent and identically
distributed (i.i.d.) Gaussian samples using the sample mean and sample
covariance is a classic example found in textbooks.
\citet[Chapter 7.1.1]{boyd2004convex} show that covariance
estimation in this standard setting and several of its variants can be formulated as convex optimization
problems.
However, many estimation problems that arise in robotics and other engineering disciplines extend
beyond the standard setting.
While \emph{noise samples} are assumed to be i.i.d.\ for each
measurement \emph{type} (Section~\ref{rem:multiple}), the measurements
themselves are \emph{not} identically distributed. Each measurement follows a Gaussian
distribution with the corresponding noise covariance matrix and a \emph{unique} mean that depends on an
unknown parameter belonging to a manifold.
Furthermore, the measurement function varies across different measurements and
is often nonlinear. 
We demonstrate that the noise covariance estimation problem in this
more general setting can also be formulated as a convex optimization problem. Similar to
\cite{boyd2004convex}, we explore several problem variants that incorporate
prior information and additional structural constraints on the noise covariance
matrix. These variants differ from those studied in
\cite[Chapter 7.1.1]{boyd2004convex} and admit analytical (closed-form)
solutions.

\subsection{Covariance Estimation in Robotics and Computer Vision}
\citet{zhan2025generalized} propose a joint pose and noise
covariance estimation method for the perspective-$n$-point (P$n$P) problem in
computer vision. Their approach is based on the iterated (or iterative)
generalized least squares (IGLS) method (see \cite[Chapter
12.5]{SeberWild200309} and references therein), alternating between pose
and noise covariance estimation. They report improvements in
estimation accuracy
ranging from $2\%$ to $34\%$ compared to baseline methods that assume a fixed
isotropic noise covariance.
To the best of our knowledge, before \cite{zhan2025generalized}, IGLS had not been
applied in robotics or computer vision.\footnote{We were
unable to find the original reference for IGLS. However, the concept was already known
and analyzed in the econometrics literature in 1970s
\cite{malinvaud1980statistical}. IGLS is closely
related to the feasible generalized least squares (FGLS) method which also has a
long history in econometrics and regression.}
Our work (developed concurrently with
\cite{zhan2025generalized}) generalizes and extends both IGLS and 
\cite{zhan2025generalized} in several keys ways.
First, we prove that the joint ML estimation problem is ill-posed when the sample covariance
matrix is singular. This critical problem arises frequently in
real-world applications, where the sample covariance matrix can be singular or
poorly conditioned (Remark~\ref{rem:singular}). We address this critical issue by
(i) constraining the minimum eigenvalue of the noise covariance matrix, and, in our MAP formulation, (ii) imposing a Wishart prior on the noise information
matrix. We derive analytical optimal solutions for the noise
covariance matrix, conditioned on fixed values of the primary parameters, for MAP
and ML joint estimation problems and several of their constrained variants
(Theorem~\ref{thm:main}). These enable the end user to leverage prior information
about the noise covariance matrix (from, e.g., manufacturer's calibration) in
addition to noisy measurements to solve the joint estimation problem. We propose several algorithms for
solving these joint estimation problems and present a rigorous theoretical
analysis of their convergence properties.  Our formulation is more general than
IGLS and we show how our framework can be extended to heteroscedastic
measurements, nonlinear (with respect to) noise models, and manifold-valued
parameters.  Finally, we provide insights into the application of our approach
to ``graph-structured'' estimation problems \cite{khosoussi2019reliable}, such
as PGO and other SLAM variants. These problems present
additional challenges compared to the P$n$P problem due to the increasing number
of primary parameters (thousands of poses in PGO vs.\ a single pose in P$n$P) and
the sparsity of measurements, which reduce the effective signal-to-noise ratio.

\citet{barfoot2020exactly} and \citet{wong2020variational} propose an
EM-type method for learning the noise covariance matrix as part of a 
variational inference framework. Similar to our work, they used an
Inverse-Wishart prior on the covariance matrix.
However, we estimate the covariance matrix by solving the joint MAP/ML
estimation problems and their constrained variants. While the covariance
estimation formulation and proposed techniques
share similarities, our method focuses on widely used MAP/ML \emph{point estimation} rather than
obtaining an analytical approximation of the entire posterior, yielding
significantly faster solutions (by up to several orders of magnitude based on the 
statistics reported in \cite{barfoot2020exactly}). As a result, unlike
\cite{barfoot2020exactly,wong2020variational}, our
method can be readily integrated into existing nonlinear least squares solvers such as
\cite{Agarwal_Ceres_Solver_2022,gtsam,kuemmerle2011g2o} in both online
and offline settings with a negligible computational overhead.

\citet{barfoot2024state} proposes an algorithm for joint state and covariance MAP
estimation (without our structural constraints) in which each measurement is
assigned a unique covariance. In \cite[Chapter 5.5.3]{barfoot2024state}, the conditionally optimal
covariances are derived and eliminated analytically from the optimization
problem. The resulting reduced objective resembles the Cauchy robust cost, which
is minimized via iteratively reweighted least squares. Our Elimination algorithm
(Algorithm~\ref{alg:elimination}) extends this approach to the more general
setting where arbitrary subsets of measurements can share a covariance matrix.
In this case, the reduced objective includes log-determinant terms and no longer
corresponds to the Cauchy cost. We therefore directly optimize the reduced objective using the
limited-memory BFGS method. We also propose a second algorithm, Block Coordinate
Descent (Algorithm~\ref{alg:hybrid-BCD}), which is preferable to Elimination as
it offers faster convergence, exploits problem sparsity, and leverages existing
sparse nonlinear least squares solvers.

\citet{lu2022slam} introduce a covariance autotuning method for object
measurements in SLAM, employing an alternating optimization scheme over states
and the variances of a \emph{diagonal} covariance matrix. Despite
similarities between \cite{lu2022slam} and our BCD algorithm, the cost
function in \cite{lu2022slam} differs from standard MAP/ML
formulation. Additionally, unlike \cite{lu2022slam}, our approach does not assume that the
covariance matrix is diagonal.

\citet{qadri2024learning} propose a bilevel optimization framework to learn
measurement covariances from a calibration dataset with known ground
truth. Specifically, they seek the covariance estimate (outer problem)
that minimizes the state estimation (inner problem) error.
Unlike \cite{qadri2024learning}, our method does not directly minimize
the estimation error and thus does not require access to the ground
truth. Instead, we jointly estimate both states
and covariances directly from observed data (and, optionally, prior information
on the covariance) in a joint MAP/ML framework. As a result, our method
can be used during deployment to learn measurement covariances based on
the collected measurements.

\subsection{Noise Covariance Estimation in Kalman Filtering}
The problem of identifying process and measurement noise models in Kalman
filtering (often referred to as \emph{adaptive} Kalman filtering) has been
extensively studied since the late 1960s; see, e.g.,
\cite{abramson1970simultaneous,mehra1972approaches,zhang2020identification,chen2023kalman,forsling2024matrix}
and references therein. While our work shares certain similarities with these
approaches and their underlying principles, such methods are specifically
designed for (approximate, when models are nonlinear) {recursive} MAP estimation in
{linear(ized)} models within the {filtering} setting. As a result,
they are not readily applicable to batch and {smoothing} formulations, nonlinear
measurement models, sparse large-scale problems, or (nonlinear)
manifold-valued states. These features are essential for addressing many
estimation problems in robotics and computer vision (see, e.g., state-of-the-art
estimation frameworks for SLAM \cite{dellaert2017factor}).

\begin{figure}[t]
		\centering
		\includegraphics[width=0.49\textwidth]{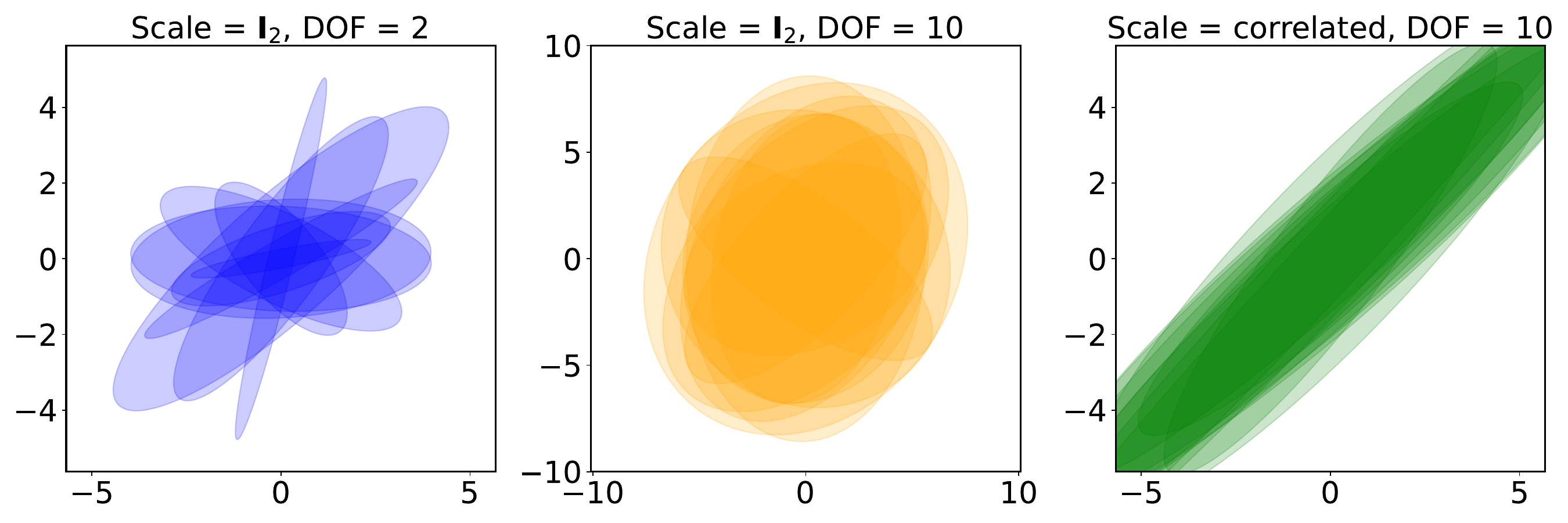}
		\caption{Confidence ellipses for 10 samples drawn from the Wishart distribution 
				$\mathcal{W}(P;V,\nu)$ with different parameters. The scale matrix $V$ is set 
				to identity in the left and middle plots, and to a correlated matrix in the 
				right plot. The degrees of freedom $\nu$ are set to $2$ (left) and $10$ (middle and right). 
				As $\nu$ increases, the samples become more concentrated. Increasing $\nu$ changes the scale of samples as
well (cf.\ left and middle).}
		\label{fig:wishart}
\end{figure}

\section{Problem Statement}
\label{sec:problem}
Consider the standard problem of estimating an unknown vector 
$\xtrue \in \Mcal$ given $k$ noisy
$m$-dimensional measurements $\Zcal \triangleq \{z_i\}_{i=1}^{k}$ corrupted by i.i.d.\ zero-mean
Gaussian noise:
\begin{align}
		z_i  = h_i(\xtrue) \boxplus \epsilon_i,\quad  \epsilon_i  \sim
		\mathcal{N}(0_m,\Sigtrue),
		\label{eq:measurement_model}
\end{align}
where $\Sigtrue$ is the unknown noise covariance.  For Euclidean-valued
measurements (such as relative position of a landmark with respect to a robot
pose), $\boxplus$ reduces to addition in $\Rset^m$.  For matrix Lie group-valued
measurements (such as relative pose or orientation between two poses),
$\boxplus$ is equivalent to multiplication by $\mathrm{Exp}(\epsilon_i)$ where
$\mathrm{Exp}$ denotes the matrix exponential composed with the
so-called hat operator.  We denote the \emph{residual} of
measurement $z_i$ evaluated at $x \in \Mcal$ with $r_i(x) \triangleq z_i
\boxminus h_i(x)$. As above, $\boxminus$ is subtraction in $\Rset^m$ for
Euclidean measurements, and $\mathrm{Log}(h_i(x)^{-1}z_i)$ in the case of matrix
Lie-group-valued measurements where $\mathrm{Log}$ is the matrix logarithm composed
with the so-called vee operator (in this case, $m$ refers to the dimension of Lie algebra).

In this paper, we refer to $\xtrue$ as the \emph{primary} ``parameters'' to distinguish them from
$\Sigtrue$.
To simplify the discussion, we first consider the
case where all measurements share the same noise distribution, meaning there is
a single noise covariance matrix $\Sigtrue$ in \eqref{eq:measurement_model}.
See Section~\ref{rem:multiple} for extensions to more general cases.
In robotics and computer vision applications, $\Mcal$
is typically a (smooth) product manifold comprised of
$\Rset^d$ and $\SOd(d)$ components ($d \in \{2,3\}$) and other 
real components (e.g., time offsets, IMU biases).
We assume the measurement functions $h_i : \Mcal \to \Rset^m$ are smooth.
This standard model (along with extensions in Section~\ref{sec:opt_cov_est}) is
quite general, capturing many estimation problems 
such as SLAM (with various sensing modalities and variants),
PGO, point cloud registration (with known correspondences),
perspective-$n$-point, and bundle adjustment.

In this paper, we are interested in the setting where the noise covariance
matrix $\Sigtrue \succ 0$ is 
unknown and must be estimated \emph{jointly} with $\xtrue$ based on the
collected measurements $\{z_i\}_{i=1}^{k}$. 
For convenience, we formulate the problem in the \emph{information form} and estimate the
noise information (or precision) matrix $P_\text{true} \triangleq \Sigtrue^{-1}$.
Without loss of generality, we assume a non-informative prior on $\xtrue$ which
is almost always the case in real-world applications
(effectively treating it as an unknown ``parameter'').
We assume a Wishart prior on the noise information matrix $\Ptrue$ and denote its
probability density function with $\Wcal(P;V,\nu)$ where $V \in \mathbb{S}^m_{\succ
0}$ is the scale matrix
and the integer $\nu \geq m+1$ is the number of degrees of freedom;
Figure~\ref{fig:wishart} illustrates random samples drawn from $\Wcal(P;V,\nu)$.
The Wishart distribution is the standard choice for prior in Bayesian
statistics for estimating the information matrix from
multivariate Gaussian data (in part due to conjugacy); see, e.g.,
\cite[Eq.\ (2.155)]{bishop2006pattern}.
In Algorithm~\ref{alg:mode}, we propose a procedure for setting the parameters of
the prior, $V$ and $\nu$,
when the user has access to a prior estimate $\Sigma_0$ for the noise covariance matrix $\Sigtrue$ (e.g., from prior calibration by
the manufacturer of the sensor); see Appendix~\ref{app:wishart} for a
justification.

\begin{algorithm}[t]
    \caption{Setting Wishart Parameters via Mode Matching}
    \label{alg:mode}
    \begin{algorithmic}[1] 
			\Procedure{PriorModeMatching}{$\Sigma_0$, $w_\text{prior}$, $k$}
			\State {\color{green!40!black} {// $\Sigma_0 \succ 0$ is a prior estimate
			for $\Sigtrue$}}
			\State {\color{green!40!black} {// $w_\text{prior} > 0$ is the weight assigned
			to prior relative to measurement likelihood}}
			\State {\color{green!40!black} {// $m$ is the dimension of
			$\Sigma_0$}}
			\State {\color{green!40!black} {// $k$ is the number of measurements}}
			\State $V \leftarrow (w_\text{prior} \,k\, \Sigma_0)^{-1}$
			\label{line:inversion}
			\State $\nu \leftarrow w_\text{prior} \,k + m + 1$
			\State \textbf{return} $(V,\nu)$
        \EndProcedure
    \end{algorithmic}
\end{algorithm}

The joint MAP estimator
for $\xtrue$ and $P_\text{true}$ are the maximizers of the posterior.
Invoking the Bayes' Theorem (and omitting the normalizing constant) results in the
following problem:
\begin{align}
		\underset{x \in \Mcal,P \succeq 0}{\text{maximize}} \quad
		\underbracket{\Wcal(P;V,\nu)}_{\text{\footnotesize prior on $\Ptrue$}}
		\overbracket{\prod_{i=1}^k
		\Ncal(z_i;h_i(x), P^{-1})}^{\text{\footnotesize measurement likelihood}}.
		\label{eq:mle_def}
\end{align}
For any $x \in \Mcal$, define the \emph{sample covariance at $x$} as
follows: 
\begin{align}
		S(x) \triangleq \frac{1}{k} \sum_{i=1}^k r_i(x) r_i(x)^\top \succeq 0.
		\label{eq:sample_cov}
\end{align}
Then, computing the negative log-posterior, omitting normalizing constants, dividing
the objective by $k+\nu -m -1$, and using the cyclic property of $\mathrm{trace}$ yields
the following equivalent problem.
\begin{problem}[Joint MAP]
		\begin{align}
				\underset{x \in \Mcal, P \succeq 0}{\text{minimize}}
				\,\, F(x,P) \triangleq - \log\det P + \left\langle
				M(x),P\right\rangle,
				\label{eq:negative_log_likelihood}
		\end{align}
		where
		\begin{align}
				M(x) \triangleq \frac{k S(x) + V^{-1}}{k+\nu-m-1} \succ 0.
				\label{eq:M}
		\end{align}
		\label{prob:joint}
\end{problem}
We also introduce and study three new variants of Problem~\ref{prob:joint} by
imposing additional (hard) constraints on the noise covariance matrix $\Sigma$.
These constraints enable the user to enforce prior
structural information about the covariance.

In the
first variant, $\Sigma$ (and thus $P$) is forced to be
diagonal. This allows the user to enforce independence between noise components.
\begin{problem}[Diagonal Joint MAP]
		\begin{align}
				\begin{aligned}
						& \underset{x \in \Mcal, P \succeq 0}{\text{minimize}}
						& & \,\, F(x,P)\\
						& \text{subject to}
						& & \text{$P$ is diagonal}.
				\end{aligned}
				\label{eq:negative_log_likelihood_diag}
		\end{align}
		\label{prob:joint_diagonal}
\end{problem}
In the second variant, 
we constrain the eigenvalues of $\Sigma = P^{-1}$ to
$[\lambda_\text{min},\lambda_\text{max}]$ where 
$\lambda_\text{max} \geq \lambda_\text{min} > 0$. 
These eigenvalues specify the minimum and maximum variance along all directions
in $\Rset^m$ (i.e., variance of all normalized linear combinations of noise
components), Therefore, 
this constraint allows the user to incorporate prior knowledge about the
sensor's noise limits. We will also show that in many real-world instances
(especially with a weak or no prior),
constraining the smallest eigenvalue of $\Sigma$ (or largest eigenvalue of $P$)
is essential for preventing $\Sigma$ from collapsing to zero.
\begin{problem}[Eigenvalue-constrained Joint MAP]
		\begin{align}
				\begin{aligned}
						& \underset{x \in \Mcal, P \succeq 0}{\text{minimize}}
						& & \,\, F(x,P)\\
						& \text{subject to}
						& & \lambda_\text{max}^{-1} I \preceq P
						\preceq \lambda_\text{min}^{-1} I.
				\end{aligned}
				\label{eq:negative_log_likelihood_eig}
		\end{align}
		\label{prob:joint_constrained}
\end{problem}
The last variant imposes both
constraints simultaneously.
\begin{problem}[Diagonal Eigenvalue-constrained Joint MAP]
		\begin{align}
				\begin{aligned}
						& \underset{x \in \Mcal, P \succeq 0}{\text{minimize}}
						& & \,\, F(x,P)\\
						& \text{subject to}
						 & & 
						 \lambda_\text{max}^{-1} I \preceq P
						 \preceq \lambda_\text{min}^{-1} I,\\
						 & & & \text{$P$ is diagonal}.
				\end{aligned}
				\label{eq:negative_log_likelihood_diag_eig}
		\end{align}
		\label{prob:joint_diagonal_constrained}
\end{problem}
\begin{remark}[Joint ML Estimation]
One can resort to joint ML estimation when a prior
distribution is not available.  The negative log-likelihood cost function
arising in the joint ML estimation problem and its constrained
variants (i.e., with eigenvalue and/or diagonal constraints) takes a form similar to $F(x, P)$ defined in
\eqref{eq:negative_log_likelihood}:
		\begin{align}
				F_\text{ML}(x,P) \triangleq - \log\det P + \langle S(x), 
				P \rangle,
				\label{eq:jointMLcost}
		\end{align}
		where $S(x)$ denotes the sample covariance matrix \eqref{eq:sample_cov}.
		The \emph{unconstrained} joint ML problem was also derived in
		\cite{zhan2025generalized}.
		\label{rem:jointMLRemark}
\end{remark}
\begin{remark}[Fixing Noise Covariance]
		It is worth noting that by fixing $P = P_0$, Problem~\ref{prob:joint} reduces to 
				 nonlinear least squares over $x$ and can
				be (locally) solved using existing solvers 
				\cite{Agarwal_Ceres_Solver_2022,gtsam,kuemmerle2011g2o}. To see
				this, note that the only term in the objective function 
				\eqref{eq:negative_log_likelihood} that is a function of $x$ can
				be written as:
				\begin{align}
						\langle M(x), P_0 \rangle &= {\gamma}^{-1}\langle
						k S(x), P_0 \rangle  + \text{const}\\
				& = {\gamma}^{-1} \left\langle \sum_{i=1}^k
				r_i(x) r_i(x)^\top , P_0 \right\rangle + \text{const} \\
				& = \gamma^{-1} \sum_{i=1}^{k} \|r_i(x)\|^2_{P_0} +
		\text{const},
						\label{<+label+>}
				\end{align}
				where $\gamma \triangleq k+ \nu -m -1$ is the constant that
				appears in the denominator of
				\eqref{eq:M}.
		\label{rem:fixedCov}
\end{remark}

\section{Optimal Information Matrix Estimation}
\label{sec:opt_cov_est}
Problems~\ref{prob:joint}-\ref{prob:joint_diagonal_constrained} are in general
\emph{non}-convex in $x$ because of the residuals $r_i$'s and, when estimating
rotations (e.g., in SLAM), the $\SOd(d)$ constraints imposed on rotational
components of $x$. In this section, we reveal a convexity structure in these
problems and provide analytical (globally) optimal solutions for estimating the
covariance matrix for a given $x \in \Mcal$. 
\subsection{Inner Subproblem: Covariance Estimation}
Problems~\ref{prob:joint}-\ref{prob:joint_diagonal_constrained} can be separated
into two nested subproblems: an inner subproblem and an outer subproblem, i.e.,
\begin{align}
		\begin{aligned}
				 \underset{x \in \Mcal}{\text{minimize}} & \qquad \underset{P
						 \succeq
				0}{\min}
				& & \,\, F(x,P)\\
				& \quad\quad \text{subject to}
				& & 
				\text{appropriate constraints on $P$},
		\end{aligned}
		\label{eq:joint_information}
\end{align}
where the constraints for each problem are given in \problems{}. The inner subproblem focuses on minimizing the
objective function over the information matrix $P$ and as a function of $x \in \Mcal$. The outer subproblem minimizes the overall objective
function by optimizing over $x \in \Mcal$ when the objective is evaluated at the optimal information matrix
obtained from the inner subproblem. 
\begin{remark}\label{rem:interior}
For \eqref{eq:joint_information} to be well defined, we require the constraint
set to be closed and the cost function to be bounded from below. As $M(x)$ is
positive definite by construction, the cost is bounded from below. The positive
semidefinite constraint guarantees the closedness of the constraint set.
However, due to the
presence of the $\log\det P$ term in the cost function, if a solution exists,
the cone constraints are not active at this solution. Thus, a singular $P^\star$
can never be a solution to this problem.
\end{remark}
\subsection{Analytical Solution to the Inner Subproblem}
The inner subproblem for a fixed $x$ can be written as
\begin{align}
		\begin{aligned}
				& \underset{P \succeq 0}{\text{minimize}}
				& & -\log\det P +  \langle M(x), P \rangle\\
				& \text{subject to}
				& & 
				\text{appropriate constraints on $P$}.
		\end{aligned}
		\label{eq:inner}
\end{align}
\begin{proposition}
		For any $x \in \Mcal$, the inner subproblem \eqref{eq:inner} is a convex optimization problem
		and has at most one optimal solution.
		\label{prop:convexity}
\end{proposition}
\begin{proof}
		See Appendix~\ref{app:prop}.
\end{proof}

The following theorem provides analytical expressions for the unique optimal solutions
to the inner subproblems \eqref{eq:inner} in \problems{}.
\begin{theorem}[Analytical Solution to the Inner Problem]
		Consider the inner problem \eqref{eq:inner} for a given $x \in \Mcal$.
				The following statements hold:
				\begin{enumerate}[leftmargin=*]
				\item {Inner Subproblem in Problem~\ref{prob:joint}}: 
					The optimal solution is given by
						\begin{align}
								P^\star(x) = M(x)^{-1}.
								\label{eq:first_statement}
						\end{align}
				\item {Inner Subproblem in Problem~\ref{prob:joint_diagonal}}:
						The optimal solution is given by
						\begin{align}
								P^\star(x) = \mathrm{Diag}(M(x))^{-1}
								\label{eq:Pstar_diag}
						\end{align}
				\item {Inner Subproblem in Problem~\ref{prob:joint_constrained}}: 
						Let 
						\begin{equation}
								M(x) = U(x)D(x)U(x)^\top
						\end{equation}
						be an eigendecomposition
						of $M(x)$ where $U(x)$ and
						$D(x)$ are orthogonal and diagonal, respectively.
						The optimal solution is given by
						\begin{align}
								P^\star(x) = U(x) \Lambda(x) U(x)^\top
								\label{eq:Pstar_joint_constrained}
						\end{align}
						where $\Lambda(x)$ is a diagonal matrix with the following
						elements:
						\begin{align}
								\Lambda_{ii}(x) = 
								\begin{cases}
										\lambda_\text{\normalfont min}^{-1} & 
										D_{ii}(x) \in [0,
										\lambda_\text{\normalfont min}],
										\\
										D_{ii}(x)^{-1} & 
										D_{ii}(x) \in (\lambda_\text{\normalfont min},
										\lambda_\text{\normalfont max}),\\
										\lambda_\text{\normalfont max}^{-1} &
										D_{ii}(x) \in
										[\lambda_\text{\normalfont max},\infty).
								\end{cases}
						\end{align}
				\item {Inner Subproblem in Problem~\ref{prob:joint_diagonal_constrained}}: 
						The optimal solution is a diagonal matrix with the
						following elements:
						\begin{align}
								P^\star_{ii}(x) =
								\begin{cases}
										\lambda_\text{\normalfont min}^{-1} & 
										M_{ii}(x) \in [0,
										\lambda_\text{\normalfont min}],
										\\
										 M_{ii}(x)^{-1} & 
										M_{ii}(x) \in (\lambda_\text{\normalfont min},
										\lambda_\text{\normalfont max}),\\
										\lambda_\text{\normalfont max}^{-1} &
										M_{ii}(x) \in
										[\lambda_\text{\normalfont max},\infty).
								\end{cases}
								\label{eq:Pstar_both}
						\end{align}
		\end{enumerate}
		\label{thm:main}
\end{theorem}
\begin{proof}
		See Appendix \ref{sec:proof_main}.
\end{proof}

As we saw in Remark~\ref{rem:jointMLRemark}, the objective function in
the ML formulation for estimating the noise covariance can be written as
$-\log\det P + \langle S(x), P \rangle$ which has a similar form to
the cost function in \eqref{eq:inner}. However, unlike $M(x)$, the sample covariance $S(x)$
can become singular. Therefore, Theorem~\ref{thm:main} readily applies
to the ML case (and its constrained variants) with an important
exception: without the constraint $\Sigma \succeq \lambda_{\text{min}}
I$, if $S(x)$ becomes singular, the problem becomes unbounded from below
and thus ML estimation becomes ill-posed. This is formally proved in the
following theorem.

\begin{theorem}
		Consider the (unconstrained) joint ML estimation problem (Remark~\ref{rem:jointMLRemark}),
		\begin{align}
				\normalfont
				\underset{x \in \Mcal, P \succeq 0}{\text{minimize}}
				\,\, - \log\det P + \left\langle
				S(x),P\right\rangle.
				\label{}
		\end{align}
		If $S(x)$ is singular, this problem (and the corresponding ``inner
		subproblem'') is unbounded below and thus does not
		have a solution.
		\label{thm:singular}
\end{theorem}
\begin{proof}
		See Appendix \ref{sec:proof_singular}.
\end{proof}

\begin{remark}
		Theorem~\ref{thm:singular} shows that, without the Wishart prior, if
		$S(x)$ is singular, the unconstrained joint estimation problem
		(specifically, without $\Sigma \succeq \lambda_\text{min} I$)
		becomes ill-posed.  Similarly, singularity of $S(x)$ can make 
		Problems~\ref{prob:joint} and
		\ref{prob:joint_diagonal} ill-conditioned when the prior is weak (i.e.,
		$w_\text{prior}$ is small) as the corresponding objective
		function becomes very sensitive in certain directions that
		correspond to the smallest eigenvalues of $M(x)$. 
		The sample covariance matrix $S(x)$ could be singular in
		many different situations, but two particular cases that can lead to
		singularity are as follows:
		\begin{enumerate}[leftmargin=*]
				\item When $k < m$ (i.e., there are not enough measurements to
						estimate $\Sigtrue$), $S(x)$ will be singular at any $x
						\in \Mcal$.
				\item Specific values of $x \in \Mcal$ can lead to
						singularity in some problems. For example, consider the problem of
						estimating odometry noise covariance matrix in 
						PGO. Let $x_\text{odo}$ be
						the odometry estimate obtained from composing odometry
						measurements. At $x = x_\text{odo}$, the residuals
						$r_i(x_\text{odo})$ (and thus $S(x_\text{odo})$) will be
						zero.
		\end{enumerate}
		\label{rem:singular}
\end{remark}

Theorem~\ref{thm:main} has a clear geometric
interpretation.
For instance, the optimal covariance matrix $\Sigma^\star(x) = {P^\star(x)}^{-1}$ for
Problem~\ref{prob:joint_constrained} \eqref{eq:Pstar_joint_constrained} has the following properties: (i) it
preserves the eigenvectors of $M(x)$, which define the principal axes of
the associated confidence ellipsoid; (ii) it matches $M(x)$ along
directions corresponding to eigenvalues within the range
$[\lambda_\text{min}, \lambda_\text{max}]$; and (iii) it only adjusts 
the radii of the ellipsoid along the remaining principal axes to satisfy
the constraint. This is visualized in Figure~\ref{fig:ellipse}.
Similarly, the confidence ellipsoid associated to \eqref{eq:Pstar_diag} 
is obtained by projecting $M(x)$ onto the set of diagonal matrices (i.e.,
axis-aligned ellipsoids). Finally, in the case of \eqref{eq:Pstar_both}, the radii
of the axis-aligned ellipsoid are adjusted
to satisfy the eigenvalue constraint.

\begin{figure}[tbp] 
  \centering 
  \includegraphics[width=0.30\textwidth]{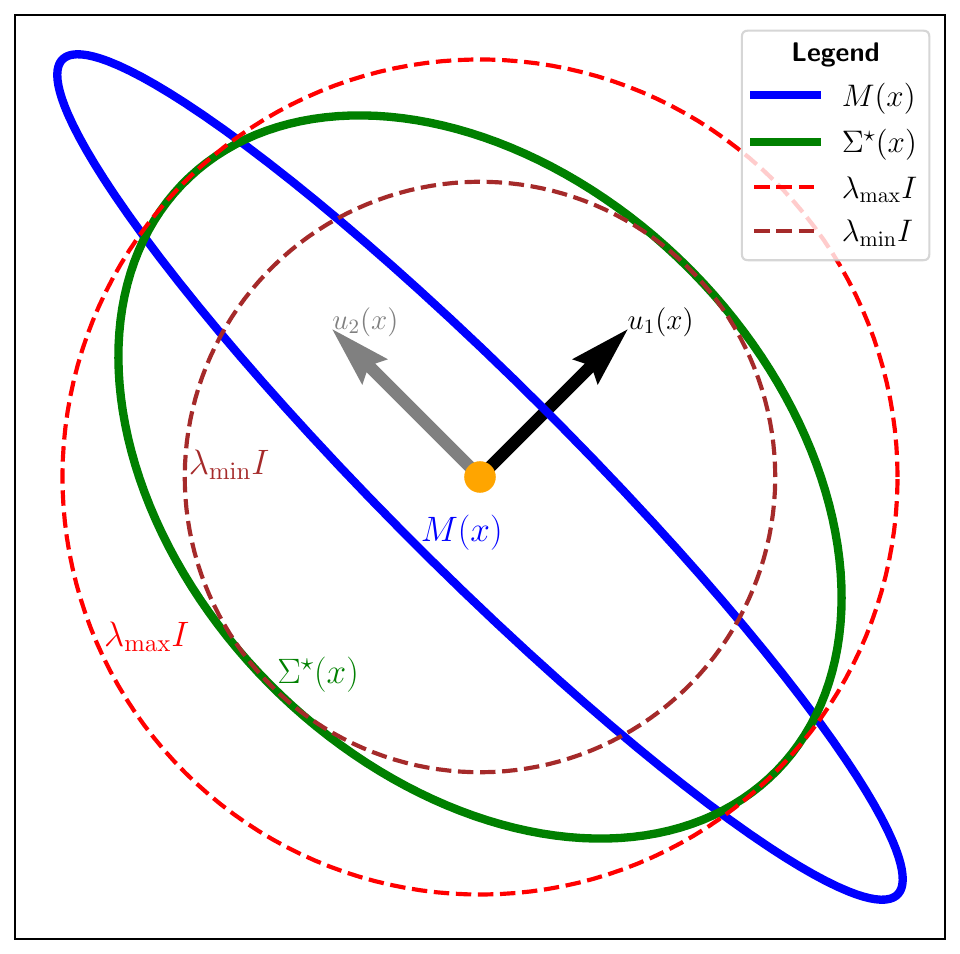} 
  \caption{The confidence ellipses corresponding to 
  $M(x)$ (in \textcolor{blue}{blue}) and the optimal covariance
  matrix $\Sigma^\star(x) = {P^\star(x)}^{-1}$ as given in
  \eqref{eq:Pstar_joint_constrained} for Problem~\ref{prob:joint_constrained}
  (in \textcolor{green!50!black}{green}). The circles show the confidence
  ellipses associated to $\lambda_\text{min}I$ and $\lambda_\text{max}I$.  Note
  that the principal axes $u_1(x)$ and $u_2(x)$ of $M(x)$ remain unchanged,
  while its radii (along the principal axes) are adjusted to fit
  within the bounds of the circles.}
  \label{fig:ellipse} 
\end{figure}

\begin{remark}[Calibration]
		The inner subproblem \eqref{eq:inner} also arises in offline calibration. In this
		context, a calibration dataset $\Zcal^\text{cal}$ is provided with ground truth $\xtrue^\text{cal}$ (or a close approximation). The
		objective is to estimate the noise covariance matrix $\Sigtrue$ based on
		the calibration dataset, which can then
		be used for future datasets collected with the same sensors. 
		Note that this approach differs from
		the joint  problems (Problems~\ref{prob:joint}-\ref{prob:joint_diagonal_constrained}),
		where $\xtrue$ and $\Sigtrue$ must be estimated \emph{simultaneously} without
		the aid of a calibration dataset containing ground truth information.
		\label{rem:calib}
		Applying
		Theorem~\ref{thm:main} at $x = \xtrue^\text{cal}$ directly provides the optimal
		noise covariance matrix for this scenario. 
		If calibration is performed
		using an approximate ground truth $x^{\text{cal}} \approx
		\xtrue^{\text{cal}}$, the estimated noise covariance matrix will be
		biased.
\end{remark}

\subsection{Two Important Extensions}
\subsubsection{\normalfont \textbf{Heteroscedastic Measurements}}
\label{rem:multiple}
		For simplicity, we have so far assumed that measurement noises are
		identically distributed (i.e., homoscedastic). However, in general, there may be $T$ distinct
		\emph{types} of measurements (e.g., obtained using different sensors in
				sensor fusion problems, odometry vs. loop closure in
		PGO, etc), where the noises corrupting each type are
		identically distributed. In such cases, the inner problem
		\eqref{eq:inner} 
		decomposes into $T$ independent problems (with different $M(x)$),
		solving each yields one of the $T$ noise information 
		matrices. Theorem~\ref{thm:main}
		can then be applied independently to each of these problems to find the
		analytical optimal solution for each noise information matrix as a function of $x$. 

		\subsubsection{\normalfont \textbf{Preprocessed Measurements and Non-Additive Noise}}
		In SLAM, raw measurements are often preprocessed and transformed
		\emph{nonlinearly}
		into standard models supported by popular solvers. For instance, raw
		range-bearing measurements (corrupted by additive noise with covariance
		$\Sigtrue$) are often expressed in Cartesian coordinates. As a result,
		although the raw measurements generated by the sensor have identically
		distributed noise,
		the \emph{transformed measurements} that appear
		in the least squares problem may have different covariances because of
		the nonlinear transformation. 
		Let $\Sigtrue$ be the covariance of raw measurements, and $\Sigma_i$ be
		the covariance of the $i$th transformed measurement. In practice,
		$\Sigma_i$ is approximated by linearization, i.e., $\Sigma_i \approx J_i
		\Sigtrue J_i^\top$ in which $J_i$ is the (known) Jacobian of the transformation.
		It is easy to verify that Theorem~\ref{thm:main} readily extends to this
		case when $J_i$'s are full-rank square matrices by replacing
		the sample covariance $S(x)$ as defined in \eqref{eq:sample_cov} with 
		\begin{align}
		   	\widetilde{S}(x) \triangleq \frac{1}{k} \sum_{i=1}^k J_i^{-1} r_i(x) r_i(x)^\top
		   	J_i^{-\top}. 
		\end{align}
		A similar technique can be used when measurements are affected by
		zero-mean Gaussian noise in a \emph{nonlinear} manner, i.e., $z_i =
		h_i(\xtrue,\epsilon_i)$ (in that case, the
		Jacobians of $h_i$'s with respect to noise will in general depend on $x$).
		\label{rem:preproc}

\section{Algorithms for Joint Estimation}
\label{sec:alg}
In principle, one can employ
existing constrained optimization methods to directly solve (locally)
\problems{}. However, this requires substantial modification of
existing highly optimized solvers such as
\cite{Agarwal_Ceres_Solver_2022,gtsam,kuemmerle2011g2o} and thus may not be
ideal. In this section, present two types of algorithms
that leverage Theorem~\ref{thm:main} for solving the joint estimation problems.

\subsection{Variable Elimination}
\begin{algorithm}[t]
    \caption{Variable Elimination for Joint MAP}
    \label{alg:elimination}
    \begin{algorithmic}[1] 
			\Procedure{VariableElimination}{}
			\State Use a local optimization method to find
			\begin{align}
				 x^\star \in \argmin_{x \in \Mcal}
				 -\log\det P^\star(x) +  \langle M(x), P^\star(x)
						 \rangle
					\nonumber
			\end{align}
			\State \textbf{return} $\big(\xmle, P^\star(\xmle)\big)$
        \EndProcedure
    \end{algorithmic}
\end{algorithm}
We can eliminate $P$ from the joint problems \eqref{eq:joint_information} by plugging in the optimal
information matrix (as a function of $x$) $P^\star(x)$ for the inner subproblem
\eqref{eq:inner} provided in Theorem~\ref{thm:main}.
This leads to the following \emph{reduced} optimization problem in $x$:
\begin{align}
				 \underset{x \in \Mcal}{\text{minimize}} \quad
				 -\log\det P^\star(x) +  \langle M(x), P^\star(x)
						 \rangle.
		\label{eq:joint_eliminated}
\end{align}
Therefore, if $x^\star \in \Mcal$ is an optimal solution for the above problem, then
$x^\star$ and $P^\star(x^\star)$ are MAP estimates for $\xtrue$
and $\Ptrue$, respectively. This suggests the simple procedure outlined in
Algorithm~\ref{alg:elimination}. Note that the reduced problem \eqref{eq:joint_eliminated} (like \problems{}) is a \emph{non}-convex
optimization problem, and thus the first step in Algorithm~\ref{alg:elimination} is subject to local minima.

\begin{remark}[Reduced Problem for Problems~\ref{prob:joint}
		and \ref{prob:joint_diagonal}]
		The objective function of the reduced problem
		\eqref{eq:joint_eliminated} further simplifies in the case of
		Problems~\ref{prob:joint} and \ref{prob:joint_diagonal}.
		Specifically, for any $x \in \Mcal$, the linear term in the objective
		(i.e., $\langle M(x),
		P^\star(x) \rangle$) is constant for the values of $P^\star(x)$ 
		given in \eqref{eq:first_statement} and \eqref{eq:Pstar_diag}; 		\begin{align}
				P^\star(x) = M(x)^{-1} &\Rightarrow \langle M(x), P^\star(x)
				\rangle = m, \\
				P^\star(x) = \mathrm{Diag}\big(M(x)\big)^{-1} &\Rightarrow
				\langle M(x), P^\star(x)
				\rangle = m.
				\label{<+label+>}
		\end{align}
		Therefore, in these cases the reduced problem further simplifies to the
		following:
		\begin{align}
				 \underset{x \in \Mcal}{\text{minimize}} & \quad 
				 \log\det M(x), \\
				 \underset{x \in \Mcal}{\text{minimize}}  & \quad
				 \log\det \mathrm{Diag}(M(x)).
				 \label{<+label+>}
		\end{align}
		These simplified problems have an intuitive geometric interpretation (as
		noted in \cite{zhan2025generalized} for the unconstrained ML estimation
case): the MAP estimate of $\xtrue$ is the value of $x \in \Mcal$ that minimizes
the volume of confidence ellipsoid (in Problem~\ref{prob:joint}) and the volume
of the projected (onto the standard basis) ellipsoid (in
Problem~\ref{prob:joint_diagonal}) characterised by $M(x)$.
		\label{rem:volume}
\end{remark}

The variable elimination algorithm has the following practical drawbacks:
\begin{enumerate}[leftmargin=*]
		\item In many real-world estimation problems, the
				residuals $r_i$ are typically sparse, meaning each measurement
				depends on only a small subset of elements in $x$. Exploiting
				this sparsity is essential for solving large-scale problems
				efficiently. However, this sparse structure is
				generally lost after eliminating $P$ in the reduced problem
				\eqref{eq:joint_eliminated}.  
		\item Popular solvers in robotics and computer vision such as \cite{Agarwal_Ceres_Solver_2022, gtsam, kuemmerle2011g2o} 
				are
				primarily nonlinear least squares solvers that assume $\Sigtrue$
				is known and focus solely on optimizing $x$. Consequently, these
				highly optimized tools cannot be directly applied to solve the
				reduced problem in \eqref{eq:joint_eliminated}. 
\end{enumerate}

\subsection{Block-Coordinate Descent}
\label{sec:BCD}
In this section we show how block-coordinate descent (BCD) methods can be
used to solve the problem of interest. A BCD-type algorithm alternates between the
following two steps until a stopping condition (e.g., convergence) is satisfied:
\begin{enumerate}[leftmargin=*]
		\item 	Fix $P$ to its most recent value
				and minimize the joint MAP objective function in
				\problems{} with respect to $x
				\in \Mcal$. This results in a standard nonlinear
				least squares problem (where residuals are weighted by $P$) over $\Mcal$, which can be (locally)
				solved using existing solvers such as
				\cite{Agarwal_Ceres_Solver_2022, gtsam, kuemmerle2011g2o} (see
				Remark~\ref{rem:fixedCov}).
		\item   Fix $x$ to its most recent value and
				minimize the joint MAP objective function with respect to $P
				\succeq 0$, subject to the constraints in \eqref{eq:inner}. This
				step reduces to solving the inner subproblem  for which we have
				analytical optimal solutions $P^\star(x)$ provided by Theorem~\ref{thm:main}.
\end{enumerate} 
Two variants of this procedure are shown in Algorithms~\ref{alg:hybrid-BCD} and
\ref{alg:BCD}.
In Step~1 of Algorithm~\ref{alg:hybrid-BCD}, $R(x, P) \triangleq
\langle M(x), P \rangle$ denotes the component of the joint cost function $F(x,
P)$ that depends on $x$. 
As demonstrated in Remark~\ref{rem:fixedCov}, minimizing this function with
respect to $x$ for a fixed $P$ is equivalent to minimizing the associated
weighted nonlinear least squares objective.
Moreover, while Algorithm~\ref{alg:hybrid-BCD}
uses Riemannian gradient descent to update $x$ in Step 1, in principle one can
use any (trust-region or line-search) optimization method that produces a
descent iteration such as those already implemented in
\cite{Agarwal_Ceres_Solver_2022, gtsam, kuemmerle2011g2o}.\footnote{We analyze
		the convergence properties of Algorithm~\ref{alg:hybrid-BCD} in
Theorem~\ref{thm:hybrid-BCD} when Riemannian gradient descent is used in Step
1.}. Algorithm~\ref{alg:BCD} is applicable to problems where in Step~1 of BCD
one can exactly minimize
$F(x,P)$ over $x$ and find the (unique)
minimizer $x^t$ for the current value of $P$.

The BCD algorithms of the type considered here address the limitations of Algorithm~\ref{alg:elimination}.
Specifically, the problem in the first step can be readily solved using standard
solvers widely used in robotics and computer vision, which are also capable
of exploiting sparsity in residuals. The second step is highly efficient, as
it only requires computing the $m \times m$ matrix $M(x)$ as defined in
\eqref{eq:M},   
which can be done in $O(k \cdot m^2)$ time.
In the case of Problem~\ref{prob:joint_constrained}, one must also compute the 
eigendecomposition of $M(x)$.
In practice, $m$ (the dimension of the residuals) is typically a small constant
(i.e., $m = O(1)$; e.g., in 3D PGO, $m = 6$),
and therefore the overall time complexity of the second step of BCD is $O(k)$,
i.e., linear in the number of
measurements. In a 2D PGO problem (i.e., $m=3$) with $k=5{,}598$ measurements, updating the
information matrix in Step~2 takes about one millisecond on a laptop
CPU.

\begin{algorithm}[t]
    \caption{Hybrid Block-Coordinate Descent for Joint MAP}
    \label{alg:hybrid-BCD}
    \begin{algorithmic}[1] 
			\Procedure{BCD}{$x_\text{init}$}
			\State $t\leftarrow 0$
			\State $x^0 \leftarrow x_\text{init}$
			\State {\color{green!40!black} {// Initialize the information
			matrix}}
			\State $P^0 \leftarrow P^\star(x_\text{init})$
            \While{$t\leq \tau$}
			\State $t\leftarrow t+1$
			\State {\color{green!40!black} {// Step 1: Update
			$x^t$  using a descent step using retraction
			$\text{Retr}_{x^{t-1}}(\cdot)$, Riemannian
			gradient $\mathrm{grad}_x \, R(\cdot,\cdot)$ with
			respect to $x$, and step-size $\eta$ where $R(x,P) \triangleq
			\langle M(x),P \rangle$; see
Theorem~\ref{thm:hybrid-BCD} and Appendix~\ref{sec:geo_back}.} }
\State $x^t\leftarrow \text{Retr}_{x^{t-1}}\left(-\eta \,
\mathrm{grad}_x \, R (x^{t-1},P^{t-1}) \right)$
\State {\color{green!40!black} {// Step 2: optimize $P$
			(Theorem~\ref{thm:main})}}
			\State $P^t \leftarrow P^\star(x^t)$ \label{line:covupdate}
            \EndWhile
			\State \textbf{return} $(x^t,P^t)$
        \EndProcedure
    \end{algorithmic}
\end{algorithm}
\begin{algorithm}[t]
    \caption{Block-Exact BCD for Joint MAP}
    \label{alg:BCD}
    \begin{algorithmic}[1] 
			\Procedure{BCD}{$x_\text{init}$}
			\State $t\leftarrow 0$
			\State {\color{green!40!black} {// Initialize the information
			matrix}}
			\State $P^0 \leftarrow P^\star(x_\text{init})$
            \While{not converged}
			\State $t\leftarrow t+1$
			\State {\color{green!40!black} {// Step 1: optimize $x$}}
			\State $x^t \in \argmin_{x \in \Mcal} \,\, \frac{1}{2} \sum_{i=1}^k
			\|r_i(x)\|^{2}_{P^{t-1}}$
			\State {\color{green!40!black} {// Step 2: optimize $P$
			(Theorem~\ref{thm:main})}}
			\State $P^t \leftarrow P^\star(x^t)$
            \EndWhile
			\State \textbf{return} $(x^t,P^t)$
        \EndProcedure
    \end{algorithmic}
\end{algorithm}

Before studying the convergence properties of these algorithms we introduce the
necessary assumption below. See Appendix~\ref{sec:geo_back} for background
information.
\begin{assumption}\label{ass:basic}
		Let $\mathcal{P}$ denote the constraint set for the noise information matrix $P$.
		The sets $\mathcal{M}$ and $\mathcal{P}$ are closed and nonempty and the function $F$ is differentiable and its level set $\{(x,P) : F (x,P) \leq\gamma\}$
		is bounded for every scalar $\gamma$.
\end{assumption}
\begin{assumption}[Lipschitz Smoothness in $x$]\label{ass:blip}
		Function $R(x,P)$ is continuously differentiable and Lipschitz smooth in $x$, i.e., there exists a
		positive constant $L$ such that for all $(\xi,\Pi)\in \mathcal{M}\times
		\mathcal{P}$ and all
		$\zeta\in\mathcal{M}$:
		\begin{equation}
				\|\nabla_x R (\xi,\Pi) - \nabla_x R (\zeta,\Pi)\|\leq L \|\xi-\zeta\|. 
		\end{equation}
\end{assumption}
Next, we state the convergence result for Algorithm~\ref{alg:hybrid-BCD} when applied to
Problems \ref{prob:joint_constrained} and \ref{prob:joint_diagonal_constrained}.
\begin{theorem}\label{thm:hybrid-BCD}
	Let $\mathcal{M}$ and $\mathcal{P}$ be compact submanifolds of the
	Euclidean space. Under Assumptions~\ref{ass:basic} and \ref{ass:blip} and setting
	$\eta=1/\widetilde{L}$ with $\widetilde{L}$ defined in
	Lemma~\ref{lem:constants} in Appendix~\ref{sec:geo_back}, for the
	sequence $\{(x^t,P^t)\}$ generated by Algorithm~\ref{alg:hybrid-BCD} we have
	\begin{equation}
			\min_{t\in [\tau]} \|\mathrm{grad} \, F(x^t,P^t)\| \leq
			C \sqrt{
					\dfrac{F(x_\mathrm{init},P_{\mathrm{init}}) -
			F(x^\star,P^\star)}{\tau}},
	\end{equation}
where $C=\sqrt{2 \widetilde{L}} (1 + \sqrt{2}\alpha)$ where $\alpha$ is given in
Lemma~\ref{lem:constants}.
\label{thm:conv1}
\end{theorem}
\begin{proof}
		See Appendix~\ref{app:convergence_2nd}.
\end{proof}
As can be seen in Algorithm~\ref{alg:BCD}, one might be able to solve
the optimization problems associated with each of the coordinate blocks uniquely and exactly. For
example, for the case where the residuals $r_i(x)$ are affine functions of $x$
and $\mathcal{M}$ is convex, the optimization problem associated with $x$ has a
unique solution (assuming a non-singular Hessian) and can be solved exactly. This case satisfies the following
assumption.
\begin{assumption}\label{ass:unique}
	For all
	$\Pi\in\mathcal{P}$ and all $\xi\in\mathcal{M}$, the following problems have
	unique solutions:
	\begin{equation}
			\min_{x\in\mathcal{M}} \; F(x,\Pi), \quad \min_{P\in\mathcal{P}}\;
			F(\xi,P).
	\end{equation}
\end{assumption}
\begin{theorem}
	Under Assumptions~\ref{ass:basic} and \ref{ass:unique}, the sequence $\{(x^t,P^t)\}$
	generated by Algorithm~\ref{alg:BCD} is bounded and has limit points. Moreover,
	every limit point $(x^\star,P^\star)$ is a local minimum of the
	optimization problem. 
	\label{thm:conv2}
\end{theorem}
\begin{proof}
	The proof of the theorem follows directly from \cite[Theorem
1]{peng2023block}.  
\end{proof}

\section{Experiments}
\label{sec:experiments}
\subsection{Linear Measurement Model}
\label{sec:explinear}

\begin{figure*}[htb]
    \centering

    \begin{subfigure}[t]{0.32\linewidth}
        \centering
        \includegraphics[width=\linewidth]{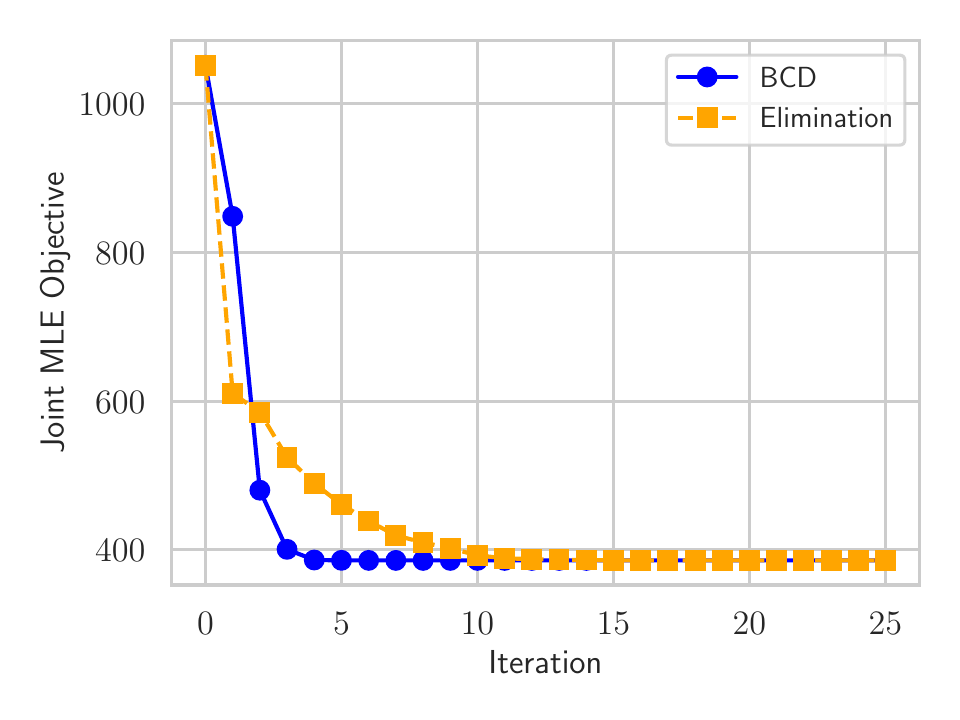}
		\caption{Objective value over iterations}
        \label{fig:objective_iterations}
    \end{subfigure}
    \hfill
    \begin{subfigure}[t]{0.32\linewidth}
        \centering
        \includegraphics[width=\linewidth]{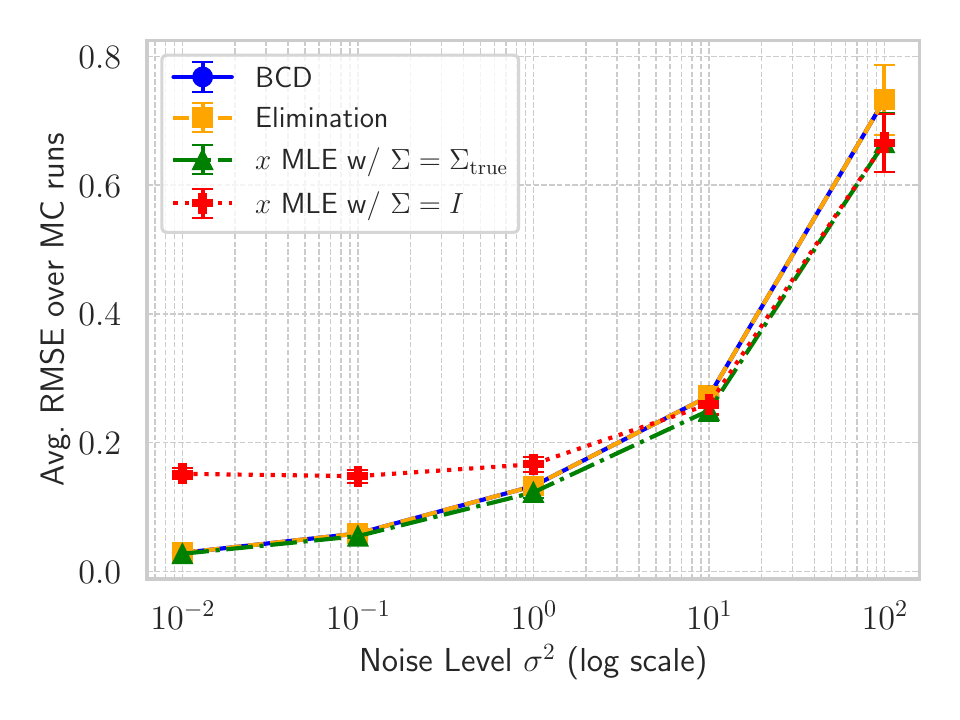}
        \caption{Average RMSE in estimating $x_\text{true}$}
        \label{fig:average_x_error}
    \end{subfigure}
    \hfill
    \begin{subfigure}[t]{0.32\linewidth}
        \centering
        \includegraphics[width=\linewidth]{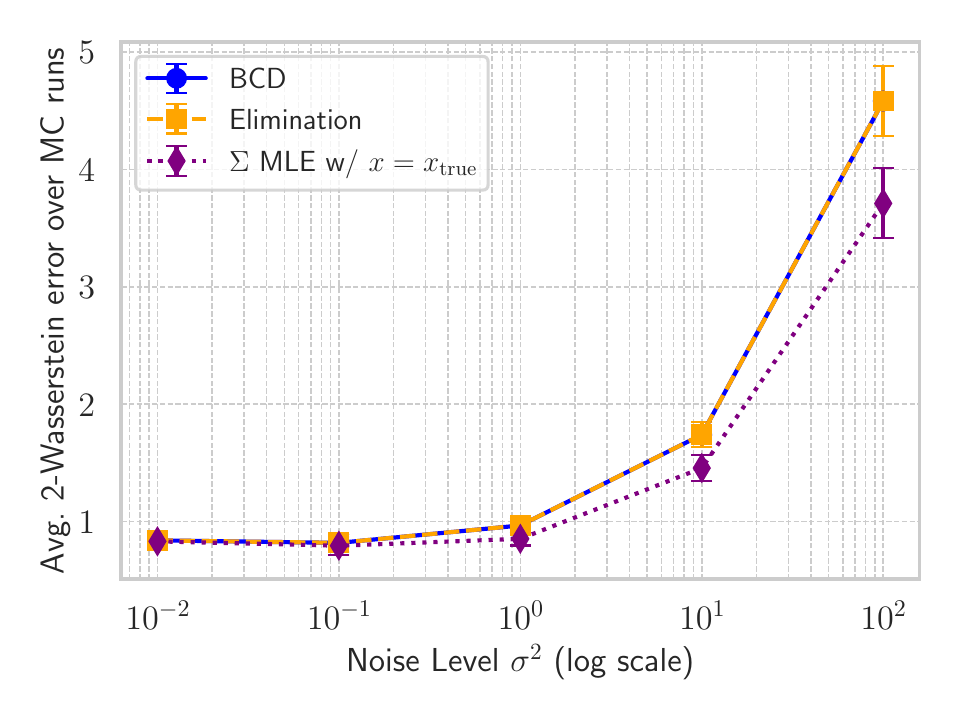}
		\caption{Average 2-Wasserstein error for noise covariance}
        \label{fig:average_C_error}
    \end{subfigure}
			\caption{Results of experiments with linear measurement models
					(Section~\ref{sec:explinear}). The results shown in
					Figures~\ref{fig:average_x_error} and \ref{fig:average_C_error} are
					averaged over $50$ Monte Carlo (MC) runs. The error bars in these
			figures represent the $95\%$ confidence intervals.}
    \label{fig:linear}
\end{figure*}

We first evaluate the algorithms on a simple linear measurement model. Although
the joint problem is still \emph{non}-convex in this scenario, the problem
is (strictly) convex in $x$ and $P$ separately, enabling exact minimization
in both steps of BCD (Algorithm~\ref{alg:BCD}) with analytical solutions (i.e., linear least squares and \eqref{eq:inner}).

\noindent\textbf{Setup:}
In these experiments, $\Mcal = \Rset^{20}$ and $\xtrue$ is set to the all-ones
vector. We generated $k = 50$ measurements, where the dimension of each
measurement is $m=5$. Measurement $z_i$ is generated according to
\begin{equation}
		z_i = H_i \xtrue + \epsilon_i, \quad i \in [k].
\end{equation}
Each measurement function $H_i \in \Rset^{5 \times 20}$ is a
random matrix drawn from the standard
normal distribution. The Measurement noise $\epsilon_i \sim \Ncal(0,\Sigtrue)$ where
\begin{equation}
		\Sigtrue = \Sigma^\text{base} + \sigma^2 I,
		\label{eq:setup}
\end{equation}
in which $\Sigma^\text{base} \succeq 0$ is a fixed random
covariance matrix, and $\sigma^2 \in
\{10^{-2},10^{-1},1,10,10^{2}\}$ is a variable that controls the noise level in our experiments.
We did not impose a prior on $P$, leading to unconstrained joint ML estimation of
$\xtrue$ and $\Sigtrue$.

We conducted 50 Monte Carlo simulations per noise level, each with a different
noise realization. In each trial, we generated measurement noise
according to the model described above and applied Elimination and BCD
(Algorithms~\ref{alg:elimination} and \ref{alg:BCD}) to estimate $\xtrue$ and
$\Sigtrue$ under the Problem~\ref{prob:joint} formulation. Both algorithms were
initialized with $x_\text{init} = 0$ and executed for up to $25$ iterations. The
Elimination algorithm uses the Limited-memory BFGS solver from SciPy
\cite{2020SciPy}.

\noindent\textbf{Metrics:}
We use the root mean square error (RMSE) to evaluate
the accuracy of estimating $\xtrue$. For covariance estimation accuracy, we
use the 2-Wasserstein distance between the true noise distribution $\Ncal(0,
\Sigtrue)$ and the estimated noise distribution $\Ncal(0, \Sigma^\star)$:
\begin{equation}
	\mathfrak{W}_2(\Ncal_\text{true},\Ncal^\star) = \sqrt{\mathrm{trace}\Big(\Sigtrue +
				\Sigma^\star - 2 \big(\Sigtrue^{\frac{1}{2}} \Sigma^\star
\Sigtrue^{\frac{1}{2}}\big)^{\frac{1}{2}} \Big)}.
		\label{eq:W2}
\end{equation}

\noindent\textbf{Results:}
\label{section:linearresuls}
The results are shown in Figure~\ref{fig:linear}.

Figure~\ref{fig:objective_iterations} illustrates the value of the objective
function in Problem~\ref{prob:joint} during one of the Monte Carlo simulations
($\sigma = 0.1$). The objective value for BCD is recorded after
updating $x$ (Step 1 of Algorithm~\ref{alg:BCD}). This figure demonstrates that
both methods eventually converge to the same objective value, although BCD
exhibits faster convergence.

Figure~\ref{fig:average_x_error} presents the average RMSE for the solution
$x^\star$, obtained by Elimination and BCD (at their final iteration), across
the Monte Carlo simulations for various noise levels. The figure also includes
average RMSEs for estimates obtained using fixed covariance matrices: $\Sigma =
\Sigtrue$ (i.e., the true noise covariance) and $\Sigma = I$ (i.e., an arbitrary
identity matrix often used in practice when $\Sigtrue$ is unknown). The results
show that Elimination and BCD achieve identical RMSEs across all noise levels,
with RMSE increasing as the noise level rises. 
Notably, under low noise, the accuracy of the
solutions produced by these algorithms matches that achieved when the true
covariance matrix $\Sigtrue$ is known. This indicates that the algorithms can
accurately estimate $\xtrue$ without prior knowledge of $\Sigtrue$. The gap
between the RMSEs widens as the noise level increases, which is expected because
the algorithms must jointly estimate $\xtrue$ and $\Sigtrue$ under a low
signal-to-noise ratio. Nonetheless, the RMSE trends consistently across noise
levels.

Furthermore, the results highlight that using a na\"{i}ve approximation of the
noise covariance matrix (i.e., fixing $\Sigma = I$) leads to a poor estimate 
$x^\star$. However, as the noise level increases, the RMSE for this
approximation eventually approaches that of the case where $\Sigtrue$ is
perfectly known. This behavior is partly due to the setup in \eqref{eq:setup}:
as $\sigma^2$ grows, the diagonal components of the covariance matrix dominate,
making the true covariance approximately isotropic. Since the estimation of
$\xtrue$ (given a fixed noise covariance matrix) is invariant to the scaling of
the covariance matrix, the performance of $\Sigma = I$ aligns with that of
$\Sigma = \Sigtrue$ despite the scaling discrepancy.

Finally, Figure~\ref{fig:average_C_error} shows the 2-Wasserstein distance
averaged over the Monte Carlo simulations. Similar to the previous figure, the
covariance estimates obtained by Elimination and BCD are consistently close to
those derived using $x = \xtrue$. 
As the noise level increases, covariance estimation error also rises, widening
the gap, as expected.

\subsection{Pose-Graph Optimization Ablations}
\noindent\textbf{Dataset:}
We used a popular synthetic PGO benchmark, the Manhattan dataset
\cite{olson2006fast}, and generated new measurement realizations with varying
values of the actual noise covariance matrix. The dataset consists of $3{,}500$
poses and $k = 5{,}598$ relative-pose measurements.  This dataset is notoriously poorly
connected \cite{khosoussi2019reliable}. Therefore, to analyze the effect of connectivity
on covariance estimation, we also performed experiments on modified versions of
this dataset, where additional loop closures were introduced by connecting pose
$i$ to poses $i+2$ and $i+3$.  This modification increased the total number of
measurements to $k = 12{,}593$.

We generated zero-mean Gaussian noise 
in the Lie algebra $\mathrm{se}(2) \cong \Rset^3$ for the following models:
\begin{enumerate}[leftmargin=*]
    \item \textbf{Homoscedastic Measurements:} 
          In these experiments, all measurements share the same information matrix, 
          given by $\alpha \times \mathrm{diag}(20, 40, 30)$, where $\alpha$ is a scaling factor that controls 
          the noise level (hereafter referred to as the ``information level'').
          
    \item \textbf{Heteroscedastic Measurements:} 
          We introduce two distinct noise models for odometry and loop-closure edges. 
          The true information matrix for odometry noise is fixed at 
          $\mathrm{diag}(1000, 1000, 800)$, while the loop-closure noise
		  information matrix 
          is varied as $\alpha \times \mathrm{diag}(20, 40, 30)$.
\end{enumerate}
For each value of the information level $\alpha \in \{5, 10, 20, 30, 40\}$, we conducted $50$ Monte Carlo simulations with
different noise realizations.

\noindent\textbf{Algorithms:}
We implemented a variant of Algorithm~\ref{alg:hybrid-BCD} in C++, based on g2o
\cite{kuemmerle2011g2o}, for PGO problems.  Instead of using Riemannian gradient
descent, we used g2o's implementation of Powell's Dog-Leg method \cite{powell1970new} on
$\mathrm{SE}(2)$ to optimize $x$.
In each outer iteration, our implementation retrieves the residuals for each
measurement, computes $M(x)$ as defined in \eqref{eq:M}, and updates the noise
covariance estimate using Theorem~\ref{thm:main}.  
We set $\lambda_\text{min} = 10^{-4}$ and
$\lambda_\text{max} = 10^4$ in all experiments to handle cases where the smallest
eigenvalue of $S(x)$ is (approximately) zero.  We then perform a single iteration of Powell's
Dog-Leg method \cite{powell1970new} to update the primary parameters $x$,
initializing the solver at the latest estimate of $x$. 
To ensure convergence across all Monte Carlo trials, we ran 13 outer iterations.
In this dataset, the per-iteration computational overhead of our
algorithm (i.e., time spent updating the covariance matrix in
Step~2 of Algorithm~\ref{alg:hybrid-BCD} relative to g2o with a given
covariance) ranged from 0.9 to 1.5 milliseconds on an Intel i7-6820HQ CPU, which
is negligible.

We report the results for the following methods:
\begin{enumerate}[leftmargin=*]
		\item \textbf{BCD}: 
				In line~\ref{line:covupdate} of
				Algorithm~\ref{alg:hybrid-BCD},
				the noise information
				matrix is estimated using the ML estimate (i.e., no prior)
				subject to eigenvalue constraints. This is equivalent to
				\eqref{eq:Pstar_joint_constrained} after replacing $M(x)$ with
				the sample covariance $S(x)$ at the current value for $x$.
		\item \textbf{BCD (diag)}: In line~\ref{line:covupdate} of
				Algorithm~\ref{alg:hybrid-BCD}, the noise
				information matrix is estimated using the ML estimate (as above), subject
				to both diagonal and eigenvalue constraints.
				This is equivalent to
				\eqref{eq:Pstar_both} after replacing $M(x)$ with
				the sample covariance $S(x)$ at the current value for $x$.
		\item \textbf{BCD (Wishart)}: 
				In line~\ref{line:covupdate} of
				Algorithm~\ref{alg:hybrid-BCD},
				the information
				matrix is updated according to \eqref{eq:Pstar_joint_constrained}; i.e., using the MAP estimate 
				with a Wishart prior and under the eigenvalue constraints.
		\item \textbf{BCD (diag+Wishart)}: 
				In line~\ref{line:covupdate} of
				Algorithm~\ref{alg:hybrid-BCD}, the noise information
				matrix is updated according to \eqref{eq:Pstar_both}; i.e., using the MAP estimate 
				with a Wishart prior and under the diagonal and eigenvalue constraints.
\end{enumerate}
For BCD (Wishart) and BCD (diag+Wishart) where a Wishart prior was used, we applied our
Algorithm~\ref{alg:mode} to set the prior parameters (i.e., $V$ and $\nu$),
for a prior weight of $w_\text{prior} = 0.1$ and a prior estimate of $\Sigma_0
= 0.002 I$ for both odometry and loop-closure edges. Note that this prior
estimate is far from the true noise covariance value.

Additionally, we report results for estimating $x$ using Powell's Dog-Leg
solver in g2o under two fixed noise covariance settings: (i) the true covariance
matrix (as a reference) and (ii) the identity matrix (a common ad hoc
approximation used in practice).  To ensure convergence across all trials, we
performed eight iterations for these methods. 
All algorithms were initialized using a spanning tree to compute $x_\text{init}$
\cite{kuemmerle2011g2o}.

\noindent\textbf{Metrics:}
We use RMSE to measure error in the estimated robot trajectory (positions).
In all cases, the first pose is fixed to the origin and therefore aligning the
estimates with the ground truth is not needed. We also use the 2-Wasserstein
distance \eqref{eq:W2} to measure the covariance estimation error.

\begin{figure}[ht]
		\centering
		\begin{subfigure}[t]{0.4\textwidth}
				\centering
				\includegraphics[width=\textwidth]{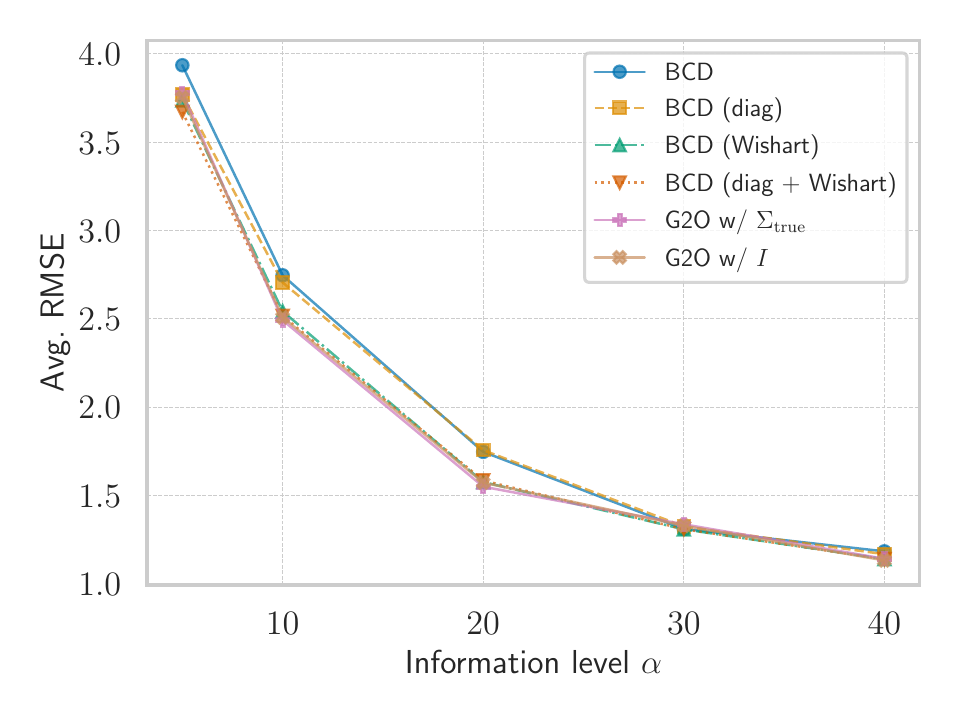}
				\caption{Homoscedastic Scenario}
				\label{fig:rmse1}
		\end{subfigure}
		\\
		\begin{subfigure}[t]{0.4\textwidth}
				\centering
				\includegraphics[width=\textwidth]{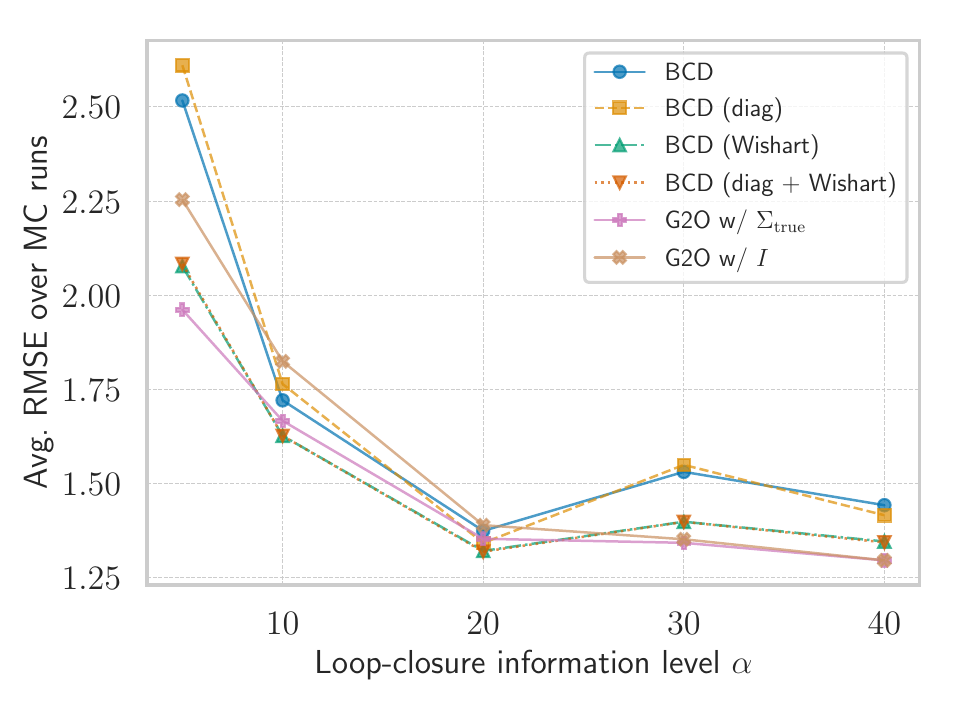}
				\caption{Heteroscedastic Scenario}
				\label{fig:rmse2}
		\end{subfigure}
		\\
		\begin{subfigure}[t]{0.4\textwidth}
				\centering
				\includegraphics[width=\textwidth]{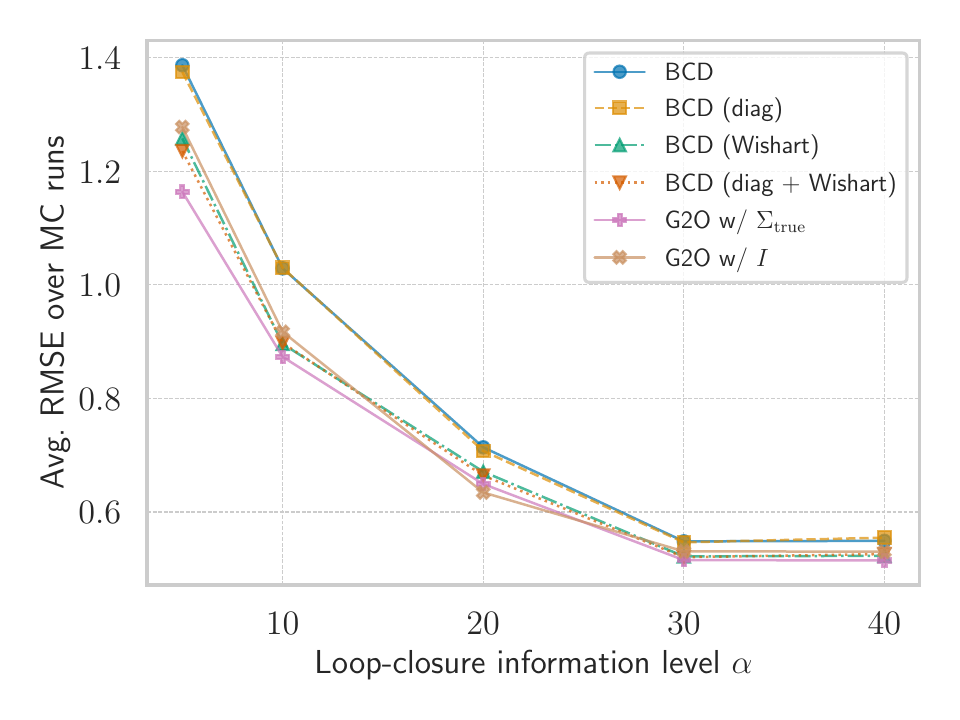}
				\caption{Heteroscedastic Scenario with extra loop closures}
				\label{fig:rmse2+}
		\end{subfigure}
		\caption{Average RMSE obtained by variants of BCD and g2o (with fixed true
		and identity covariances) as a function of information level $\alpha$.}
		\label{fig:rmses}
\end{figure}

\begin{figure*}
		\begin{subfigure}[t]{0.4\textwidth}
				\centering
				\includegraphics[height=5.5cm]{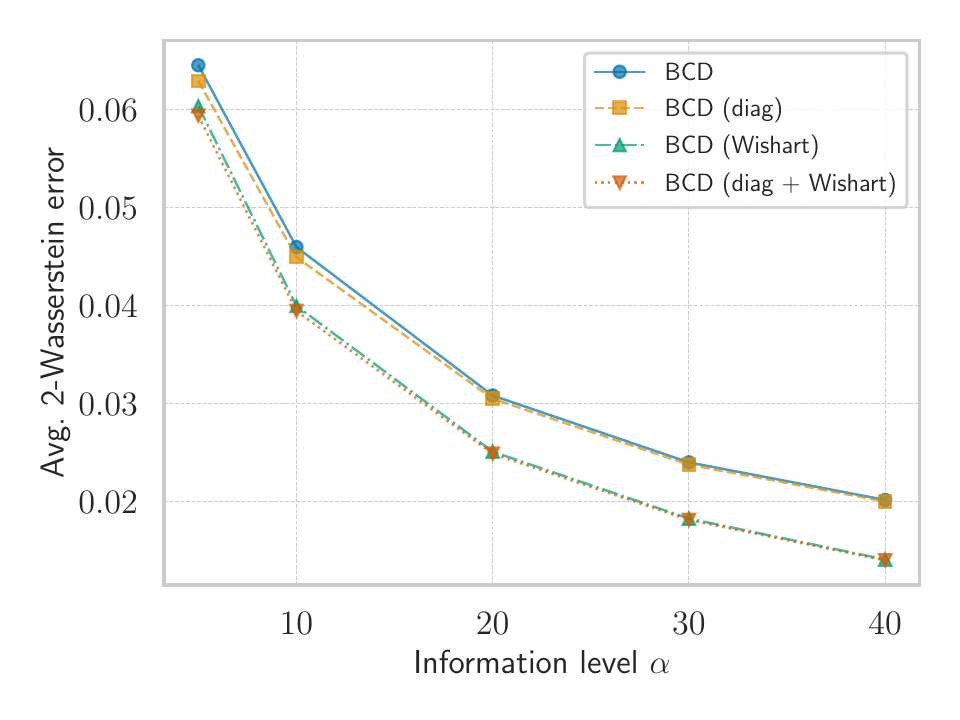}
				\caption{Homoscedastic Scenario}
				\label{fig:w21}
		\end{subfigure}
		\begin{subfigure}[t]{0.65\textwidth}
				\includegraphics[height=5.5cm]{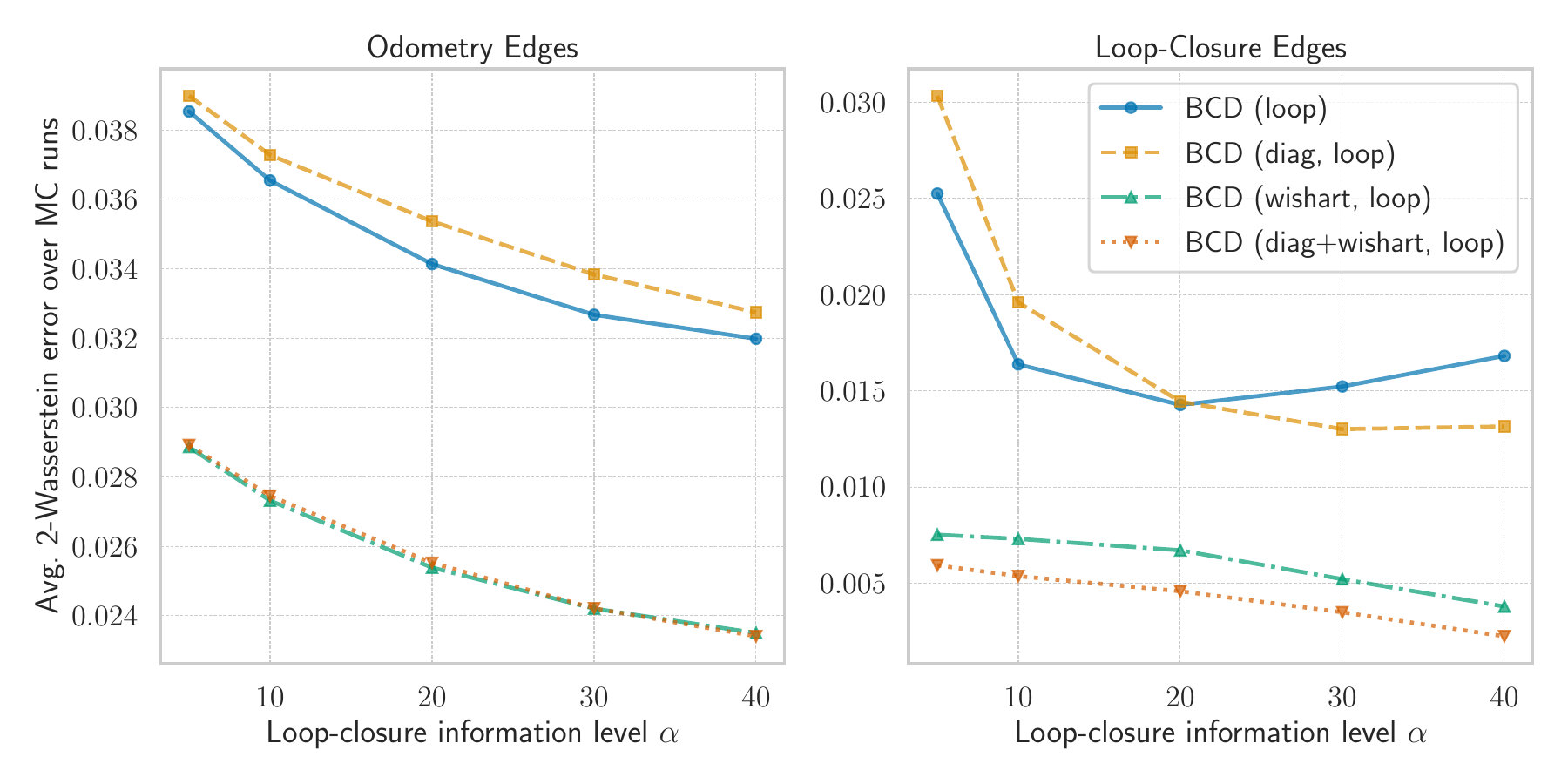}
				\caption{Heteroscedastic Scenario -- odometry (middle)
				and loop closure (right)}
				\label{fig:w22}
		\end{subfigure}
		\caption{Average 2-Wasserstein error achieved by variants of BCD and as
		a function of information level $\alpha$.}
		\label{fig:w2s}
\end{figure*}

\noindent\textbf{Results:}
The results are shown in Figures~\ref{fig:rmses} and \ref{fig:w2s}.

Figure~\ref{fig:rmses} shows the average RMSE over Monte Carlo trials for
different values of the information level \( \alpha \). 
Figure~\ref{fig:rmse1} presents the results for the homoscedastic case,
while Figures~\ref{fig:rmse2} and \ref{fig:rmse2+} show the results for the
heteroscedastic case before and after additional loop closures, respectively.
The results show that our framework, across all variants, succesfully produced solutions
with an RMSE close to the reference setting where the true noise covariance matrix
was available (g2o with \( \Sigma_{\text{true}} \)).
For almost all values of $\alpha$ and all experiments, BCD (Wishart) and BCD (diag+Wishart) achieved a lower average RMSE
compared to the other variants, highlighting the importance of incorporating a
prior. 
Notably, this is despite the prior being assigned a small weight \(
w_{\text{prior}} \) and the fact that the prior estimate \( \Sigma_0 \) is
not close to \( \Sigma_{\text{true}} \).  Incorporating a prior is particularly
crucial in PGO, as we observed that the eigenvalues of \( S(x) \) can become
severely small in practice.  In such cases, without a prior and without
enforcing a minimum eigenvalue constraint \( \Sigma \succeq \lambda_{\text{min}}
I \), the ML estimate for the noise covariance can collapse to zero, leading to
invalid results and excessive overconfidence.  The average RMSE achieved by
those variants with diagonal and without constraints are generally similar.
Interestingly, the RMSE achieved by the diagonal variants appears to be close to
that of their counterparts without this constraint, despite the true noise
covariance being diagonal. 

For some values of $\alpha$ in
Figure~\ref{fig:rmse2}, the BCD variants with the Wishart prior achieved a lower
MSE than the reference solution. However, we expect that, on average, the
reference solution will perform better with a larger number of Monte Carlo
simulations. The RMSE trends look similar in all experiments, with the exception
of a slight increase for $\alpha = 30$ for BCD and BCD (diag) in
Figure~\ref{fig:rmse2}.
In terms of RMSE, BCD variants with a Wishart
prior outperform or perform comparably to the baseline with a fixed identity
covariance in most cases. That said, especially under larger information levels
(i.e., lower noise),
this na\"{i}ve baseline estimates $x$ quite accurately.

Figure~\ref{fig:w2s} presents the average 2-Wasserstein distance \eqref{eq:W2}
between the noise covariance matrices estimated by various BCD variants and the
true covariance in both homoscedastic and heteroscedastic scenarios. 
In all cases, BCD variants achieved a small 2-Wasserstein error. For reference,
the 2-Wasserstein error between $\Ncal(0,\Sigtrue)$ and the baseline
$\Ncal(0,I)$ exceeds 1, which is more than 20 times the errors attained by our
algorithms. As noted in Section~\ref{sec:introduction}, an incorrect noise
covariance leads to severe overconfidence or underconfidence in the estimated
state $x$ (e.g., the estimated trajectory in SLAM), potentially resulting in
poor or catastrophic decisions.

The results indicate that, in most cases, diagonal BCD variants yield lower
errors. This is expected, as the true noise covariance matrices are diagonal,
and enforcing this prior information enhances estimation accuracy. Additionally,
the MAP estimates obtained by BCD (Wishart) and BCD (diag+Wishart) outperform
their ML-based counterparts. This is due to the fact that the
eigenvalues of the sample covariance matrix $S(x)$ can be very small, indicating
insufficient information for accurate noise covariance estimation. In such
cases, the eigenvalue constraint effectively prevents a complete collapse of the
estimated covariance. 
Interestingly, this issue was not encountered in \cite{zhan2025generalized} for ML
estimation of the noise covariance matrix in P$n$P. This discrepancy may suggest
additional challenges in estimating noise covariance in SLAM (and related)
problems, which typically involve a larger number of variables and sparse,
graph-structured relative measurements. In such cases, incorporating a prior
estimate may be necessary for accurate identification of noise covariance
matrices. 
Finally, Figure~\ref{fig:w22} illustrates that the MAP estimation error is
lower for the noise covariance of loop closures compared to odometry edges. This
difference is partly because the prior value based on $\Sigma_0$ is closer to
the true noise covariance matrix for loop closures.

\subsection{RIM Dataset}
\begin{figure*}[h]
		\centering
		\begin{subfigure}[t]{0.3\textwidth}
				\centering
				\includegraphics[width=\textwidth]{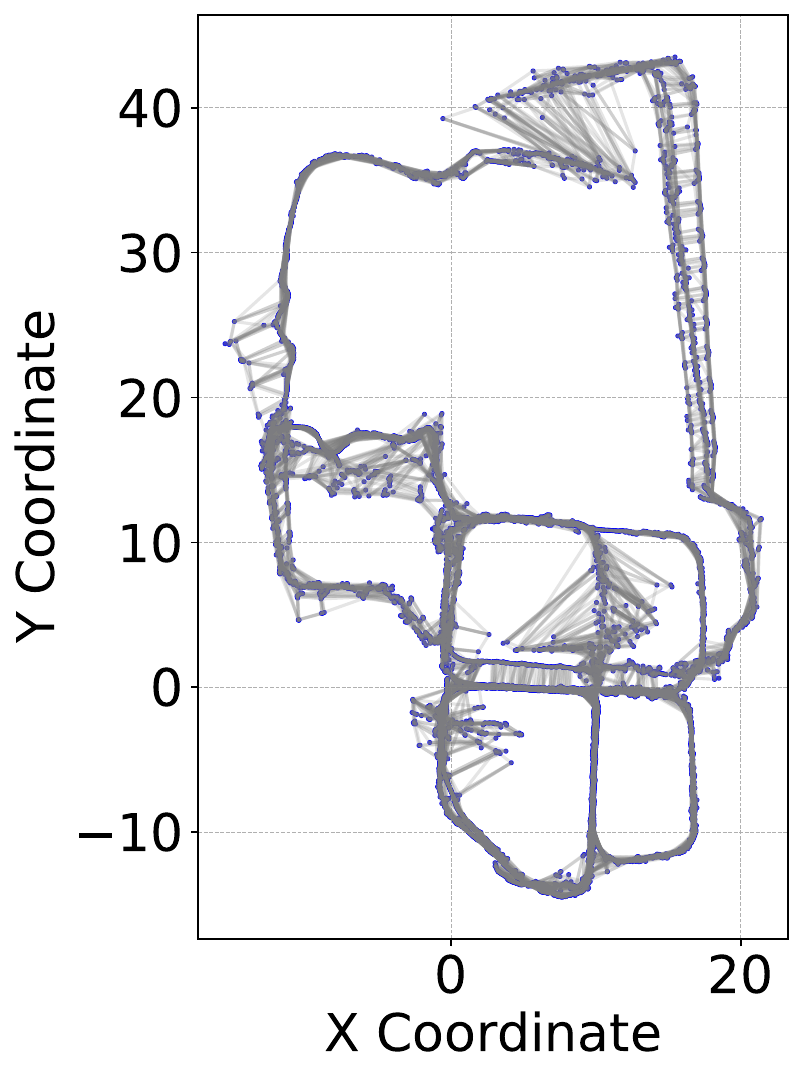}
				\caption{g2o w/ original covariance}
				\label{fig:rimg2o}
		\end{subfigure}
		\hfill
		\begin{subfigure}[t]{0.3\textwidth}
				\includegraphics[width=\textwidth]{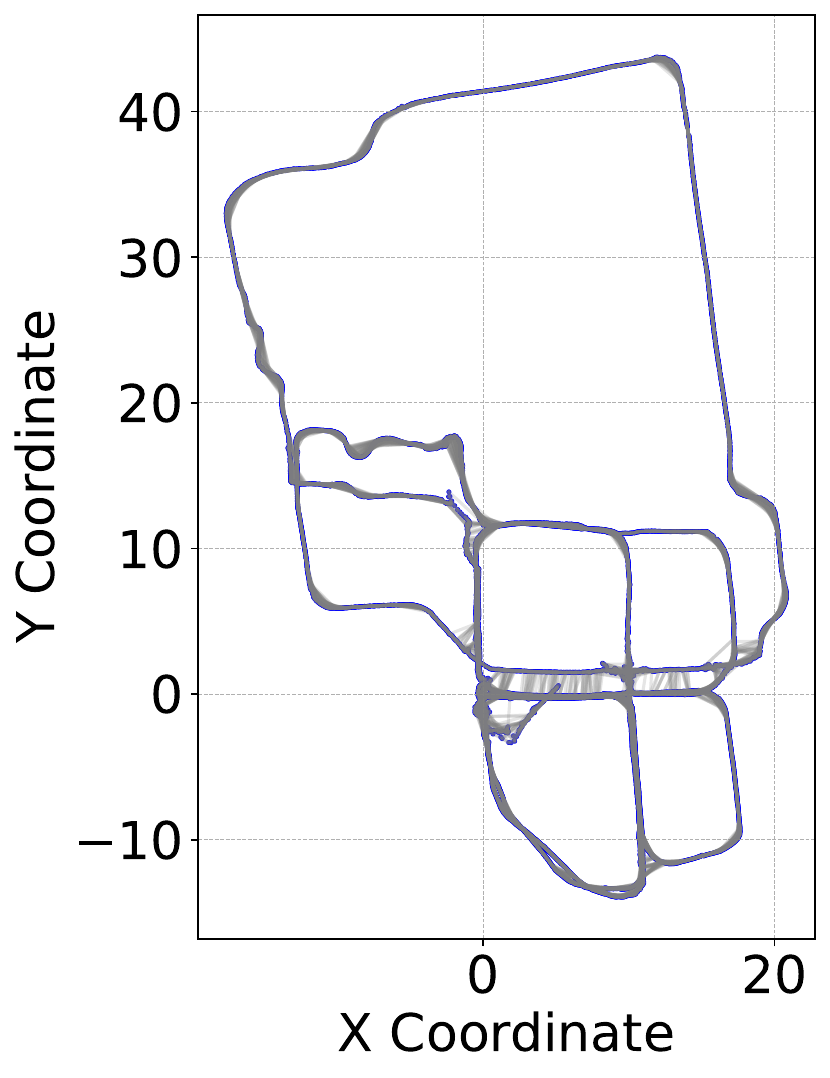}
				\caption{BCD -- joint ML estimate (w/o prior)}
				\label{fig:rimml}
		\end{subfigure}
		\hfill
		\begin{subfigure}[t]{0.3\textwidth}
				\includegraphics[width=\textwidth]{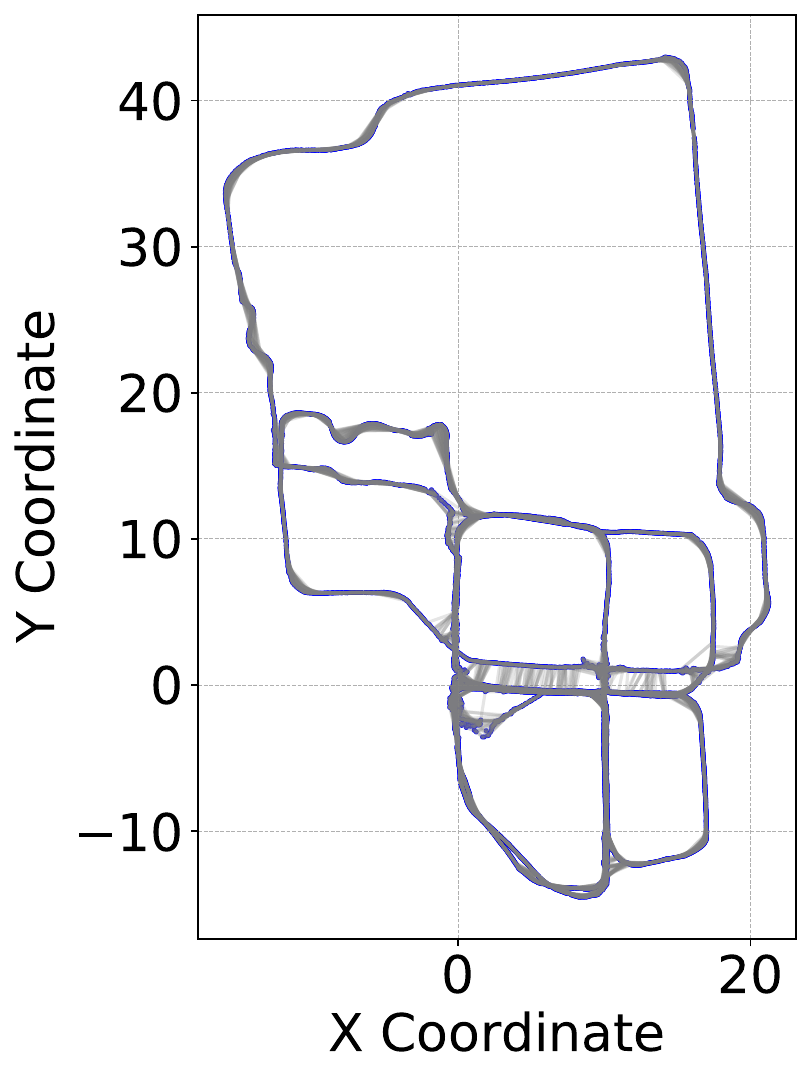}
				\caption{BCD -- joint MAP estimate w/ prior}
				\label{fig:rimmap}
		\end{subfigure}
		\caption{RIM Dataset}
		\label{fig:rim}
\end{figure*}
In this section, we present qualitative results on RIM, a large real-world 3D
PGO dataset collected at Georgia Tech \cite{carlone2015initialization}. This
dataset consists of $10{,}195$ poses and $29{,}743$ measurements. The g2o dataset
includes a default noise covariance matrix, but with these default values, g2o
struggles to solve the problem. Specifically, the Gauss-Newton and Dog-Leg
solvers terminate after a single iteration, while the Levenberg-Marquardt solver
completes 10 iterations but produces the trajectory estimate shown in
Figure~\ref{fig:rimg2o}. 

We ran BCD without and with the Wishart prior (i.e., ML and MAP estimation) for
10 outer iterations, using a single Dog-Leg iteration to optimize $x$ in each
round. We set the prior parameter based on $w_\text{prior} = 0.1$ and $\Sigma_0
= 0.01 I$. The results are displayed in Figures~\ref{fig:rimml} and
\ref{fig:rimmap}, respectively.
It is evident that the trajectory estimates obtained by BCD are significantly
more accurate than those obtained using the original noise covariance values. 

A closer visual inspection reveals that the trajectory estimated using the
Wishart prior (right) appears more precise than the ML estimate (middle). This
is expected, as without imposing a prior, some eigenvalues of $S(x)$ collapsed
until they reached the eigenvalue constraint $\Sigma \succeq \lambda_\text{min}
I$. In contrast, in the MAP case, the estimated covariance did not reach this
lower bound. 

Interestingly, despite $\Sigma_0$ being isotropic, the estimate obtained by BCD
with the Wishart prior was not. In particular, the information matrix component
associated with the $z$ coordinate was almost twice as large as those of $x$ and
$y$. This likely reflects the fact that the trajectory in this dataset is
relatively flat. This finding highlights that in this case, in addition to the
prior, the measurements also contributed information to the estimation of the
noise covariance matrix.

\section{Limitations}
\label{sec:limit}
Algorithm~\ref{alg:hybrid-BCD}, in its current form, is not robust to outliers, which limits its
applicability in real-world scenarios. However, existing robust M-estimation
techniques can be readily adapted for use with BCD. These methods iteratively
reweight residuals based on a chosen robust cost function. Our preliminary
results indicate that incorporating these weights into Step~1 of BCD (as is
standard) and also into Step~2 (by replacing the sample covariance matrix $S(x)$
with a weighted average of residuals) yields an outlier-robust variant of BCD
for the joint estimation problem. We are currently evaluating this approach and
aim to implement it in real-time LiDAR and visual SLAM and odometry systems such
as \cite{ORBSLAM3TRO,liosam2020shan}.

Additionally, Theorem~\ref{thm:hybrid-BCD} requires $\Mcal$ to be
compact. This result therefore does not immediately apply to problems where $x$
contains translational components. However, we believe we can address this
issue in SLAM and many other geometric estimation problems 
by confining translational components to potentially large but bounded subsets
of $\Rset^{d}$ (where $d \in \{2,3\}$ is the problem dimension).

Finally, we acknowledge the extensive research on IGLS and Feasible GLS
(FGLS) conducted as early
as 1970s in econometrics and statistics, which remain
largely unknown in engineering; see relevant references in \cite[Chapter 12.5]{SeberWild200309} and
\cite{zhan2025generalized}.
Exploring this rich literature may result in new insights and methods 
that can further improve noise covariance estimation in SLAM and computer vision
applications. 

\section{Conclusion}
\label{sec:conlusion}
This work presented a novel and rigorous framework for \emph{joint}
estimation of primary parameters and noise covariance matrices in 
SLAM and related computer vision estimation problems.
We derived
analytical expressions for (conditionally) optimal noise covariance matrix (under ML
and MAP criteria) under various
structural constraints on the true covariance matrix. Building on these
solutions, we proposed two types of algorithms for finding the optimal estimates and
theoretically analyzed their convergence properties. Our results
and algorithms are quite general and can be readily applied to a broad range of
estimation problems across various engineering disciplines.
Our algorithms were validated through extensive experiments using linear measurement models
and PGO problems. 
The results show that the state and the noise covariance matrix can be
jointly estimated from the measurements (and, optionally, a prior
on the covariance) with negligible computational overhead relative to standard
solvers.

\bibliographystyle{plainnat}
\bibliography{bib/slam,bib/optimization,bib/slamLoopClosure}

\appendices

\section{Wishart Prior}
\label{app:wishart}
Consider $\Wcal(P; V,\nu)$ as a prior for the information matrix. 
Here we assume $\nu \geq m+1$ where $m$ is the dimension of covariance matrix
and $V \succ 0$.
The mode of this distribution is given by
\begin{equation}
		\argmax_{P \succeq 0} \quad \Wcal(P;V,\nu) = (\nu - m - 1)V.
\end{equation}
Let $\Sigma_0 \succ 0$ be our prior guess for the covariance matrix. We set the
value of the scale matrix $V$ by matching the mode with our prior estimate $\Sigma_0^{-1}$:
\begin{align}
		(\nu - m - 1) V = \Sigma_0^{-1} \Leftrightarrow V^{-1} = 
		(\nu -m -1) \Sigma_0
		\label{<+label+>}
\end{align}
Let $w_\text{prior}$ be the following:
\begin{equation}
		w_\text{prior} \triangleq \frac{\nu -m - 1}{k} \geq 0.
		\label{eq:prior_weight}
\end{equation}
Then we can rewrite $V^{-1}$ as
\begin{equation}
		V^{-1} = w_\text{prior} k \Sigma_0.
\end{equation}
In the unconstrained case \eqref{eq:first_statement}, the optimal (MAP) estimate
(conditioned on a fixed value of $x$) is given by:
\begin{align}
		\Sigma^\star(x) \triangleq {P^\star(x)}^{-1} & = M(x) \\
		& = \frac{V^{-1} + k S(x)}{k + \nu -m -1}  \\
		& = \frac{w_\text{prior} k \Sigma_0 + k S(x)}{(w_\text{prior} + 1)k} \\
		& = \frac{w_\text{prior}}{w_\text{prior}+1} \Sigma_0 + \frac{1}{w_\text{prior}+1} S(x).
		\label{<+label+>}
\end{align}
This shows that the (conditional) MAP estimator simply blends the sample covariance
matrix $S(x)$ and prior $\Sigma_0$ based on the \emph{prior weight} $w_\text{prior}$.
We directly set $w_\text{prior}$  (e.g., $w_\text{prior} = 0.1$ by default) based on our
confidence in the prior relative to the likelihood. Large values of
$w_\text{prior}$ will result in stronger priors. For the chosen value of $w_\text{prior}$,
\begin{align}
		\nu &= w_\text{prior} k + m + 1, \\
		V &= w_\text{prior} k \Sigma_0.
		\label{<+label+>}
\end{align}
This procedure is summarized in Algorithm~\ref{alg:mode}. 

\begin{remark}[Wishart Parameters]
		\label{rem:wishartParams}
		It is worth noting that Algorithm~\ref{alg:mode} represents just one simple
		approach to setting the prior parameters. Alternative strategies can also be
		employed, such as learning the parameters from an offline calibration dataset,
		or directly specifying the value of $\nu$ independently of $k$ (with larger
		values of $\nu$ indicating greater confidence in the prior estimate
$\Sigma_0$).
		By explicitly setting a constant value for $w_{\text{prior}}$,\footnote{That is, setting
		the prior parameter $\nu$ as a function of $k$ in \eqref{eq:prior_weight}.} the
		relative influence of the prior versus the measurements remains fixed. 
		We adopted this strategy in our experiments for simplicity.
		However,
		it is important to recognize that setting $\nu$ as a function of $k$ violates
		the Bayesian principle that the prior should be specified before observing any
		data. Consequently, when $w_\text{prior}$ is kept fixed, unlike in standard MAP estimation the measurements will
		never completely dominate the prior as $k$ increases. 
\end{remark}
\section{Proof of Proposition~\ref{prop:convexity}}
\label{app:prop}
		The objective function is strictly convex because log-determinant is
		strictly concave over
		$\mathbb{S}_{\succ 0}$ and the second term is linear in $P$. The
		equality (diagonal) and linear matrix inequality (eigenvalue) constraints are also
		linear and convex, respectively.
		Since the objective is strictly convex, there is at most one optimal
		solution (i.e., if a minimizer exists, it is unique).

\section{Proof of Theorem~\ref{thm:main}}\label{sec:proof_main}
In the following, we provide primal-dual pairs that satisfy the KKT
conditions for each variant. Since the problem is differentiable and convex,
these pairs must be primal-dual optimal. Moreover, these solutions are unique
because the objective function is strictly convex.

\paragraph*{1)}For a given $x \in \Mcal$, the Lagrangian of the inner subproblem in Problem~\ref{prob:joint} is given by
\begin{equation}
		\mathcal{L}(P,Q) = -\log\det P + \langle M(x),P\rangle - \langle Q , P
		\rangle,
\end{equation}
where $Q \succeq 0$ is the Lagrange multiplier corresponding to the semidefinite cone
constraint. 
Recall that $M(x) \succ 0$ by construction. Verify that 
$P^\star(x) = M(x)^{-1}$ and $Q = 0$ trivially satisfy the KKT conditions, and
thus $P^\star(x) = M(x)^{-1}$ is the optimal solution to the primal problem.

\paragraph*{2)}For the case of the inner subproblem in Problem~\ref{prob:joint_diagonal}, the
problem can be rewritten as
\begin{align}
		\underset{\lambda_1,\dots,\lambda_m}{\text{minimize}}
	&\qquad -\sum_{i=1}^m \log \lambda_i +  
	\sum_{i=1}^m \lambda_i M_{ii}(x)
					 \nonumber \\
				  \text{subject to}
				 & 
				 \qquad \lambda_i \geq 0, \quad i\in [m].
\end{align}
The Lagrangian of this problem is given by
\begin{equation}
		\mathcal{L}(\lambda_1,\dots,\lambda_m, q_1,\dots,q_m) = \sum_{i=1}^m \left [ -\log
		\lambda_i +  \lambda_i M_{ii}(x) - \lambda_i q_i\right], 
\end{equation}
where $q_i$ are the Lagrange multipliers corresponding to the constraints.
The KKT conditions for this problem are
\begin{align}
		&-1/\lambda_i +M_{ii}(x) - q_i = 0,\quad i\in [m]\\	
	& q_i\lambda_i = 0, \quad q_i \geq 0, \quad \lambda_i \geq 0, \quad
	i\in [m].
\end{align}
It can be observed that since $M_{ii}(x)$ is positive by the virtue of $M(x)$
being a positive definite matrix, $\lambda_i^\star = {M_{ii}(x)}^{-1}$ and
$q_i^\star=0$, $\forall i\in [m]$, is the solution to the above KKT system. 

\paragraph*{3)} The Lagrangian of the inner subproblem of
Problem~\ref{prob:joint_constrained} is given by

\begin{align}
	\mathcal{L}(P,\underline{Q},\overline{Q}) &= -\log\det P + \langle M(x),
	P\rangle\\&\quad +\langle \underline{Q}, I/\lambda_{\max} - P\rangle + \langle
	\overline{Q},P- I/\lambda_{\min}\rangle. 
\end{align}
The KKT conditions corresponding to this problem read
\begin{align}
& \nabla_{P} \mathcal{L}(P,\underline{Q},\overline{Q}) =-P^{-1} + M(x) -
	 \underline{Q} +\overline{Q} = 0 \\
&\langle \underline{Q}, I/\lambda_{\max} - P\rangle = 0\\
	&\langle	\overline{Q},P-I/\lambda_{\min}\rangle = 0\\
        &\underline{Q}\succ 0,\; \overline{Q} \succ
	0,\; P=P^\top\\
	&	I/\lambda_{\max} \leq P\leq  I/\lambda_{\min}.
\end{align}
Let $\underline{Q}^\star = U(x) \underline{\Lambda} (x) U(x)^{\top}$ and
$\overline{Q}^\star = U(x) \overline{\Lambda}(x) U(x)^{\top}$ where the
elements of diagonal matrices $\underline{\Lambda}(x)$ and $\overline{\Lambda}(x)$ are 
\begin{align}
	\underline{\Lambda}_{ii}(x) = 
	\begin{cases}
		0 & D_{ii}(x) \in [0, \lambda_\text{max}],
		\\
		D_{ii}(x) - \lambda_{\max} & D_{ii}(x) \in
		(\lambda_\text{max}, \infty ),\\
	\end{cases}
\end{align}
and
\begin{align}
		\overline{\Lambda}_{ii}(x) = 
	\begin{cases}
		\lambda_{\min}- D_{ii}(x) & D_{ii}(x) \in [0, \lambda_\text{min}],
								\\
		0 & D_{ii}(x) \in
		(\lambda_\text{min}, \infty ).\\
	\end{cases}
\end{align}
It can be  observed that the choice of $P^\star(x)$ as in
\eqref{eq:Pstar_joint_constrained} along with $\overline{Q}^\star(x)$ and
$\underline{Q}^\star(x)$ defined above satisfy the KKT conditions and are the
primal-dual optimal solutions to the problem. This
completes the proof.
\paragraph*{4)} The solution to the inner subproblem of
Problem~\ref{prob:joint_diagonal_constrained} can be obtained by a reformulation
similar to case 2) above and checking the conditions similar to case 3).

\section{Proof of Theorem~\ref{thm:singular}}\label{sec:proof_singular}
Consider a strictly feasible pair $(x_0, P_0)$ where $x_0 \in \Mcal$ and $P_0
\succ 0$. Let $u_0 \in \mathbb{R}^m$ be any vector in the nullspace of $S(x_0)$.
For any $c > 0$, the pair $(x_0, P_0 + c\, u_0 u_0^\top)$ remains strictly feasible. Observe that
\begin{align}
		F_\text{ML}(x_0&,P_0+c \, u_0 u_0^\top)  = \nonumber \\ & -\log\det\left(P_0+ c \, u_0
		u_0^\top\right) + \langle S(x_0),P_0+c \, u_0 u_0^\top \rangle \\
		& = -\log\det\left(P_0+ c \, u_0
		u_0^\top\right) + \langle S(x_0),P_0 \rangle \\
		& = - \log\left(1+c \, u_0^\top P_0^{-1} u_0\right) - \log\det P_0 +\langle S(x_0),P_0
		\rangle \label{eq:proof_matrixdeterminant}\\
		& = - \log\left(1+c \, u_0^\top P_0^{-1} u_0\right) + F_\text{ML}(x_0,P_0),
		\label{<+label+>}
\end{align}
where we used the matrix determinant lemma in
\eqref{eq:proof_matrixdeterminant}.
Since $P_0 \succ 0$, we have that $u_0^\top P_0^{-1} u_0 >0$. Hence, as $c \to \infty$, we have
\[F_\text{ML}(x_0,P_0+c\, u_0 u_0^\top) \to -\infty.\] Therefore, the problem is
unbounded below and does not admit a solution.

\section{Background Information for Section \ref{sec:BCD}}\label{sec:geo_back}
Let $\Ycal$ be a smooth submanifold in Euclidean space. Each $y \in
 \Ycal$ is associated with a
 linear subspace, called the \emph{tangent space}, denoted by $T_y \Ycal$.
 Informally, the tangent space contains all directions in which one can
 tangentially pass through $y$. For a formal definition see \cite[Definition
 3.5.1, p.~34]{absil2009optimization}. A smooth function
 $F : \Ycal \rightarrow \mathbb{R}$ is associated with the (Euclidean) gradient $\nabla
F(y)$ of $F$, as well as the orthogonal projection of $\nabla F (y)$ onto the tangent space
$T_y \Ycal$, known as the \emph{Riemannian gradient} of $F$ at $y$, denoted by
$\mathrm{grad} \, F (y)$. See \cite[p.~46]{absil2009optimization} for more
detail. The manifold $\Ycal$ is also associated
with a map $\mathrm{Retr}_y : T_y \Ycal \rightarrow \Ycal$, called
a \emph{retraction} that enables moving from $y$ along the direction $v\in T_y
\Ycal$ while remaining on the manifold -- see \cite[Definition
3.47]{boumal2020intromanifolds} and \cite[Definition
4.1.1]{absil2009optimization} for formal definitions.

A sufficient condition for Assumption~\ref{ass:blip} to hold is for the manifolds to be
compact subsets of the Euclidean space. Under this assumption we have the
following lemma \cite[Lemma 1]{peng2023block}.
\begin{lemma}\label{lem:constants}
Let $\mathcal{M}$ be a compact submanifold of the Euclidean space. Then we have,
\[ \|\mathrm{Retr}_{\xi} (v) - \xi\|\leq \alpha \|v\|, \quad \forall
\xi\in\mathcal{M},\quad \forall v \in T_{\xi}{\mathcal{M}} \]
for some constant $\alpha$. Moreover, if $R$ satisfies Assumption \eqref{ass:blip}
and $\mathcal{P}$ is compact, then there exists a positive $\widetilde{L}$ such that
for all $(\xi,\Pi)\in\mathcal{M}\times\mathcal{P}$ and all $v\in
T_{\xi}{\mathcal{M}}$, we have
\begin{align}
		R (\mathrm{Retr}_{\xi}(v),\Pi) &\leq R(\xi,\Pi) + \langle \mathrm{grad}_x \, R (\xi,\Pi),v\rangle +
	\dfrac{\widetilde{L}}{2}\|v\|.
\end{align}
\end{lemma}

\section{Proof of Theorem~\ref{thm:conv1}}
\label{app:convergence_2nd}
The proof follows from specializing the proof of \cite[Theorem
4]{peng2023block} to the problem considered in this paper. Specifically,
the compactness of manifolds along with
Assumptions~\ref{ass:basic}  and \ref{ass:blip} lead to the existence of
of $\widetilde{L}$ as described in Lemma \ref{lem:constants}. Then following
the steps of the proof of \cite[Theorem 4]{peng2023block} and setting
$b=2$ completes the proof.

\end{document}